\documentclass{article} 

\usepackage{amsmath,amsfonts,bm}

\def\Figref#1{Figure~\ref{#1}}

\def\eqref#1{equation~\ref{#1}}
\def\Eqref#1{Equation~\ref{#1}}

\def\1{\bm{1}}

\def\rvd{{\mathbf{d}}}

\def\rvu{{\mathbf{i}}}

\def\rvn{{\mathbf{n}}}

\def\rvu{{\mathbf{u}}}
\def\rvv{{\mathbf{v}}}
\def\rvw{{\mathbf{w}}}
\def\rvx{{\mathbf{x}}}
\def\rvy{{\mathbf{y}}}
\def\rvz{{\mathbf{z}}}

\def\rmE{{\mathbf{E}}}

\def\mI{{\bm{I}}}

\DeclareMathAlphabet{\mathsfit}{\encodingdefault}{\sfdefault}{m}{sl}
\SetMathAlphabet{\mathsfit}{bold}{\encodingdefault}{\sfdefault}{bx}{n}

\def\gB{{\mathcal{B}}}
\def\gC{{\mathcal{C}}}
\def\gD{{\mathcal{D}}}

\def\gN{{\mathcal{N}}}

\def\gU{{\mathcal{U}}}

\def\sR{{\mathbb{R}}}

\newcommand{\E}{\mathbb{E}}

\newcommand{\R}{\mathbb{R}}

\usepackage{graphicx}
\usepackage{subfigure}
\usepackage{bigstrut, tabularx, multirow, makecell, diagbox}
\usepackage{graphicx}
\usepackage{caption}
\usepackage{zref}
\usepackage{xcolor}
\usepackage{booktabs} 
\usepackage{amssymb}
\usepackage{enumerate}
\usepackage{natbib}
\usepackage{pifont}
\usepackage{tabulary}
\usepackage[colorlinks=true,citecolor=brown,urlcolor=gray]{hyperref}
\usepackage{url}            
\usepackage{amsfonts}       
\usepackage{nicefrac}       
\usepackage{microtype}      
\usepackage{booktabs}
\usepackage{wrapfig}
\usepackage{bbm, comment}
\usepackage{color}
\usepackage{float}
\usepackage{thm-restate}
\usepackage{algorithm}
\usepackage{algorithmic}
\usepackage[utf8]{inputenc}

\newcommand{\yx}[1]{{\color{cyan!50!pink} [Yilun: #1]}}

\newcommand{\mat}[1]{\mathbf{#1}}

\newcommand{\rmin}{r_{\textrm{min}}}
\newcommand{\rmax}{r_{\textrm{max}}}
\newcommand{\tx}{\tilde{\rvx}}
\newcommand{\ty}{\tilde{\rvy}}
\newenvironment{proofs}{%
  \proof}{\endproof}

\usepackage{amsmath,amsfonts,bm,amssymb,mathtools,amsthm}
\usepackage{color,xcolor}
\usepackage{amsmath}
\usepackage{xspace} 
\def\onedot{$\mathsurround0pt\ldotp$}
\def\eg{\emph{e.g}\onedot, }

\def\ie{\emph{i.e}\onedot, }

\def\Figref#1{Fig.~\ref{#1}}

\def\eqref#1{equation~\ref{#1}}
\def\Eqref#1{Eq.~(\ref{#1})}

\def\1{\bm{1}}

\def\rvd{{\mathbf{d}}}

\def\rvu{{\mathbf{i}}}

\def\rvn{{\mathbf{n}}}

\def\rvu{{\mathbf{u}}}
\def\rvv{{\mathbf{v}}}
\def\rvw{{\mathbf{w}}}
\def\rvx{{\mathbf{x}}}
\def\rvy{{\mathbf{y}}}
\def\rvz{{\mathbf{z}}}

\def\rmE{{\mathbf{E}}}

\def\mI{{\bm{I}}}

\DeclareMathAlphabet{\mathsfit}{\encodingdefault}{\sfdefault}{m}{sl}
\SetMathAlphabet{\mathsfit}{bold}{\encodingdefault}{\sfdefault}{bx}{n}

\def\gB{{\mathcal{B}}}
\def\gC{{\mathcal{C}}}
\def\gD{{\mathcal{D}}}

\def\gN{{\mathcal{N}}}

\def\gU{{\mathcal{U}}}

\def\sR{{\mathbb{R}}}

\newcommand\numberthis{\addtocounter{equation}{1}\tag{\theequation}}

\usepackage{varwidth}
\newsavebox\tmpbox
\usepackage[export]{adjustbox}
\usepackage{caption}

\usepackage{spacingtricks}

\newcommand{\appropto}{\mathrel{\vcenter{
  \offinterlineskip\halign{\hfil$##$\cr
    \propto\cr\noalign{\kern2pt}\sim\cr\noalign{\kern-2pt}}}}}

\def\setstretch#1{\renewcommand{\baselinestretch}{#1}}
\usepackage[preprint]{icml2023}

\usepackage{amsmath}
\usepackage{amssymb}
\usepackage{mathtools}
\usepackage{amsthm}

\usepackage[capitalize,noabbrev]{cleveref}

\theoremstyle{plain}

\theoremstyle{definition}

\theoremstyle{remark}

\icmltitlerunning{PFGM++: Unlocking the Potential of Physics-Inspired Generative Models}

\begin{document}

\twocolumn[
\icmltitle{PFGM++: Unlocking the Potential of Physics-Inspired Generative Models}




\begin{icmlauthorlist}
\icmlauthor{Yilun Xu}{mit}
\icmlauthor{Ziming Liu}{mit}
\icmlauthor{Yonglong Tian}{mit}
\icmlauthor{Shangyuan Tong}{mit}
\icmlauthor{Max Tegmark}{mit}
\icmlauthor{Tommi Jaakkola}{mit}

\end{icmlauthorlist}

\icmlaffiliation{mit}{Massachusetts Institute of Technology, MIT, Cambridge, MA, USA}

\icmlcorrespondingauthor{Yilun Xu}{ylxu@mit.edu}

\icmlkeywords{Machine Learning, Generative Models, ICML}

\vskip 0.3in
]



\printAffiliationsAndNotice{}  

\begin{abstract}
We introduce a new family of physics-inspired generative models termed \textit{PFGM++} that unifies diffusion models and Poisson Flow Generative Models (PFGM). These models realize generative trajectories for $N$ dimensional data by embedding paths in $N{+}D$ dimensional space while still controlling the progression with a simple scalar norm of the $D$ additional variables. The new models reduce to PFGM when $D{=}1$ and to diffusion models when $D{\to}\infty$. The flexibility of choosing $D$ allows us to trade off robustness against rigidity as increasing $D$ results in more concentrated coupling between the data and the additional variable norms. We dispense with the biased large batch field targets used in PFGM and instead provide an unbiased perturbation-based objective similar to diffusion models. To explore different choices of $D$, we provide a direct alignment method for transferring well-tuned hyperparameters from diffusion models~($D{\to} \infty$) to any finite $D$ values. Our experiments show that models with finite $D$ can be superior to previous state-of-the-art diffusion models on CIFAR-10/FFHQ $64{\times}64$ datasets, with FID scores of $1.91/2.43$ when $D{=}2048/128$. In class-conditional setting, $D{=}2048$ yields current state-of-the-art FID of $1.74$ on CIFAR-10. In addition, we demonstrate that models with smaller $D$ exhibit improved robustness against modeling errors. Code is available at \url{https://github.com/Newbeeer/pfgmpp}
\end{abstract}

\section{Introduction}

Physics continues to inspire new deep generative models such as \textit{diffusion models}~\cite{sohl2015deep, Ho2020DenoisingDP, Song2021ScoreBasedGM, Karras2022ElucidatingTD}
based on thermodynamics~\cite{Jarzynski1997EquilibriumFD} or \textit{Poisson flow generative models}~(PFGM)~\cite{Xu2022PoissonFG} derived from  electrostatics~\cite{griffiths2005introduction}. The associated generative processes involve iteratively de-noising samples by following physically meaningful trajectories. Diffusion models learn a noise-level dependent score function so as to reverse the effects of forward diffusion, progressively reducing the noise level $\sigma$ along the generation trajectory.
PFGMs in turn augment $N$-dimensional data points with an extra dimension and evolve samples drawn from a uniform distribution over a large $N{+}1$-dimensional hemisphere back to the $z{=}0$ hyperplane where the clean data~(as charges) reside by tracing learned electric field lines. Diffusion models in particular have been demonstrated across image~\cite{Song2021ScoreBasedGM, Nichol2022GLIDETP, Ramesh2022HierarchicalTI}, 3D~\cite{Zeng2022LIONLP, Poole2022DreamFusionTU}, audio~\cite{Kong2020DiffWaveAV, Chen2020WaveGradEG} and biological data~\cite{Shi2021LearningGF, Watson2022BroadlyAA} generation, and have more stable training objectives compared to GANs~\cite{Arjovsky2017WassersteinGA, Brock2019LargeSG}. More recent PFGM \cite{Xu2022PoissonFG} rival diffusion models on image generation.  

In this paper, we introduce a broader family of physics-inspired generative models that we call \textbf{\textit{PFGM++}}. These models extend the electrostatic view into higher dimensions through multi-dimensional $\rvz \in \R^D$ augmentations. 
When interpreting $N$-dimensional data points $\rvx$ as positive charges, the electric field lines define a surjection from a uniform distribution on an infinite $N{+}D$-dimensional hemisphere to the data distribution located on the $\rvz{=}\bf{0}$ hyperplane. We can therefore draw generative samples by following the electric field lines, evolving points from the hemisphere back to the $\rvz{=}\bf{0}$ hyperplane. Since the electric field has rotational symmetry on the surface of the $D$-dim cylinder $\|\rvz\|_2 = r$ for any $r>0$, we can track the sampling trajectory with a simple scalar $r$ instead of every component of $\rvz$. The use of symmetry turns the aforementioned surjection into a bijection between an easy-to-sample prior on a large $r=\rmax$ hyper-cylinder to the data distribution. The symmetry reduction also permits $D$ to take any positive values, including reals. We derive a new perturbation-based training objective akin to denoising score matching~\cite{Vincent2011ACB} that avoids the need to use large batches to construct electric field line targets in PFGM. The perturbation-based objective is more efficient, unbiased, and compatible with paired sample training of conditional generation models.
\begin{figure*}[ht]
\centering
\includegraphics[width=0.74\textwidth, trim = 7.4cm 9.4cm 7.4cm 9.4cm, clip]{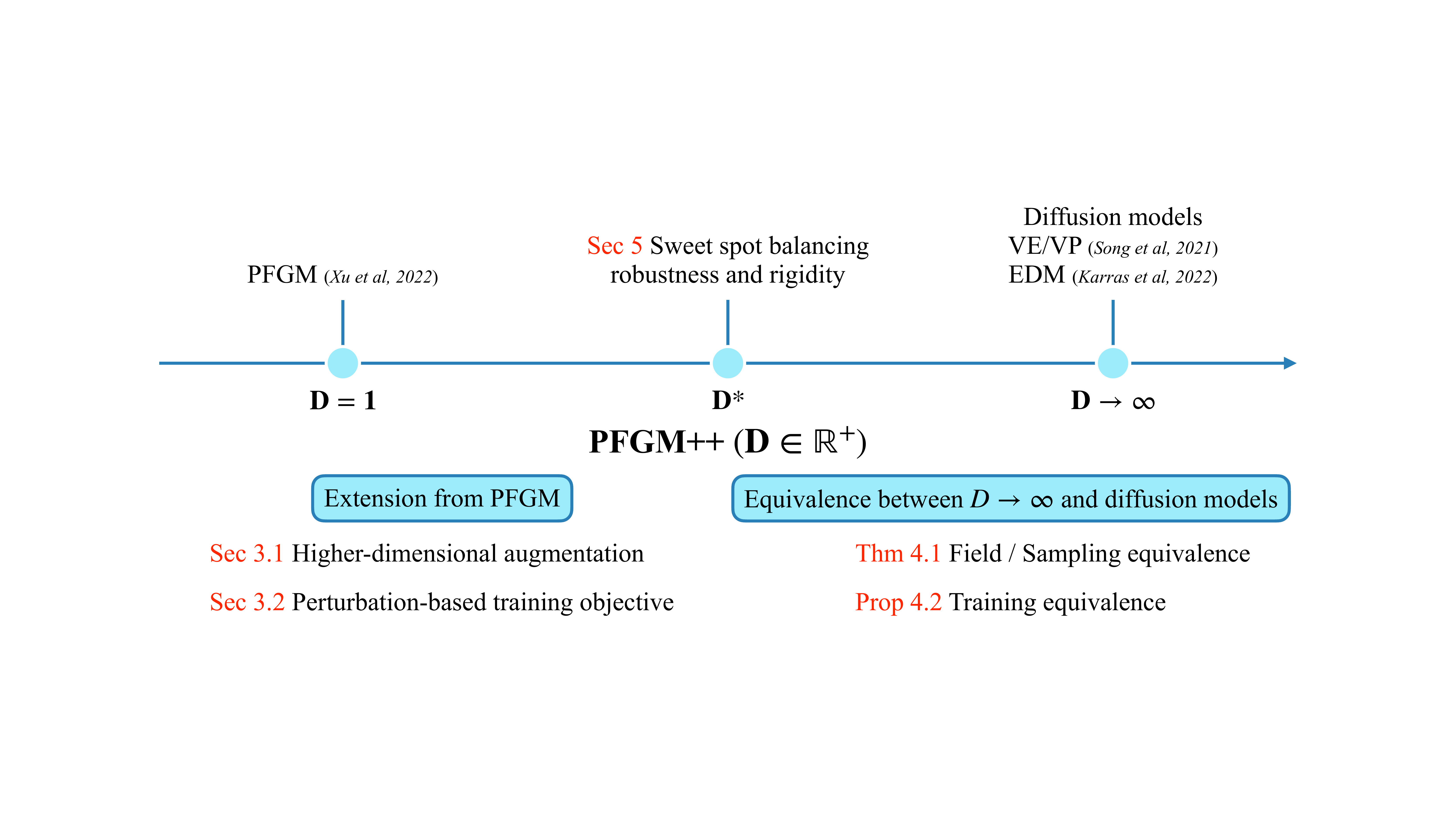}
    \caption{Overview of paper contributions and structure. PFGM++ unify PFGM and diffusion models, as well as the potential to combine their strengths (robustness and rigidity).}
    \label{fig:pfgmpp}
    \vspace{-5pt}
\end{figure*}

The models in the new family differ based on their augmentation dimension $D$ which is now a hyper-parameter. By setting $D{=}1$ we obtain PFGM while $D{\to} \infty$ leads to diffusion models. We establish $D{\to} \infty$ equivalence with popular diffusion models~\cite{Song2021ScoreBasedGM, Karras2022ElucidatingTD} both in terms of their training objectives as well as their inferential processes. We demonstrate that the hyper-parameter $D$ controls the balance between robustness and rigidity: using a small $D$ widens the distribution of noisy training sample norms in comparison to the norm of the augmented variables. However, small $D$ also leads to a heavy-tailed sampling problem at any fixed augmentation norm making learning more challenging. Neither $D{=}1$ nor $D{\to}\infty$ offers an ideal balance between being insensitive to missteps (robustness) and allowing effective learning (rigidity). Instead, we adjust $D$ in response to different architectures and tasks. To facilitate quickly finding the best $D$ we provide an alignment method to directly transfer other hyperparameters across different choices of $D$. 

Experimentally, we show that some models with finite $D$ outperform the previous state-of-the-art diffusion models~($D{\to} \infty$), \ie EDM~\cite{Karras2022ElucidatingTD}, on image generation tasks. In particular, intermediate $D{=}2048/128$ achieve the best performance among other choices of $D$ ranging from $64$ to $\infty$, with min FID scores of $1.91/2.43$ on CIFAR-10 and FFHQ $64{\times} 64$ datasets in unconditional generation, using $35/79$ NFE. In class-conditional generation, $D{=}2048$ achieves new state-of-the-art FID of $1.74$ on CIFAR-10. We further verify that in general, decreasing $D$ leads to improved robustness against a variety of sources of errors, \ie controlled noise injection, large sampling step sizes and post-training quantization.

Our contributions are summarized as follows: \textbf{(1)} We propose PFGM++ as a new family of generative models based on expanding augmented dimensions and show that symmetries involved enable us to define generative paths simply based on the scalar norm of the augmented variables~(Sec~\ref{sec:pfgmpp-intro}); \textbf{(2)} We propose a perturbation-based objective to dispense with any biased large batch derived electric field targets, allowing unbiased training~(Sec~\ref{sec:perturbed-kernel}); \textbf{(3)} We prove that the score field and the training objective of diffusion models arise in the limit $D{\to} \infty$~(Sec~\ref{sec:diffusion}); \textbf{(4)} We demonstrate the trade-off between robustness and rigidity by varying $D$~(Sec~\ref{sec:benefits}). \textit{We also detail the hyperparameter transfer procedures from EDM/DDPM~($D\to \infty$) to finite $D$s in Appendix~\ref{app:transfer-diff}}; \textbf{(5)} We empirically show that models with finite $D$ achieve superior performance to diffusion models while exhibiting improved robustness~(Sec~\ref{sec:experiment}).

\section{Background and Related Works}

\textbf{Diffusion Model}
 Diffusion models \citep{sohl2015deep, Ho2020DenoisingDP, Song2021ScoreBasedGM, Karras2022ElucidatingTD} are often presented as a pair of two processes. A fixed forward process governs the training of the model, which learns to denoise data of different noise levels. 
A corresponding backward process involves utilizing the trained model iteratively to denoise the samples starting from a fully noisy prior distribution. 
\citet{Karras2022ElucidatingTD} propose a unifying framework for popular diffusion models~(VE/VP~\cite{Song2021ScoreBasedGM} and EDM~\cite{Karras2022ElucidatingTD}), and their sampling process can be understood as traveling in time with a probability flow ordinary differential equation (ODE):
\begin{align*}
        \mathrm{d}\rvx = -\dot{\sigma}(t)\sigma(t)\nabla_\rvx \log p_{\sigma(t)}(\rvx)\mathrm{d}t
\end{align*}
where $\sigma(t)$ is a predefined noise schedule w.r.t. time, and $\nabla_\rvx\log p_{\sigma(t)}(\rvx)$ is the score of noise-injected data distribution at time $t$. A neural network $f_\theta(\rvx,\sigma)$ is trained to learn the score $\nabla_\rvx\log p_{\sigma(t)}(\rvx)$ by minimizing a weighted sum of the denoising score-matching objectives~\citep{Vincent2011ACB}:
\begin{multline*}
    \E_{\sigma\sim p(\sigma)}\lambda(\sigma)\E_{\rvy\sim p(\rvy)}\E_{\rvx\sim p_\sigma(\rvx|\rvy)}\\\left[\|f_\theta(\rvx,\sigma)-\nabla_\rvx\log p_\sigma(\rvx|\rvy)\|_2^2\right]\numberthis
    \label{eq:dm_obj}
\end{multline*}
where $p(\sigma)$ defines a training distribution of noise levels, $\lambda(\sigma)$ is a weighting function, $p(\rvy)$ is the data distribution, and $p_\sigma(\rvx|\rvy)=\gN(\mathbf{0}, \sigma^2\mI)$ defines a Gaussian perturbation kernel which samples a noisy version $\rvx$ of the clean data $\rvy$. Please refer to Table~1 in \citet{Karras2022ElucidatingTD} for specific instantiations of different diffusion models.





\textbf{PFGM}
Inspired by the theory of electrostatics~\cite{griffiths2005introduction}, \citet{Xu2022PoissonFG} propose Poisson flow generative models (PFGM), which interpret the $N$-dimensional data $\rvx \in \R^N$ as electric charges in an $N{+}1$-dimensional space augmented with an extra dimension $z$: $\tilde{\rvx} = (\rvx, z)\in\R^{N+1}$. In particular, the training data is placed on the $z{=}0$ hyperplane, and the electric field lines emitted by the charges define a bijection between the data distribution and a uniform distribution on the infinite hemisphere of the augmented space~\footnote{In practice, the  hemisphere is projected to a hyperplane $z{=}z_{\rm max}$, so that all samples have the initial $z$.}. To perform generative modeling, PFGM learn the following high-dimensional electric field, which is the derivative of the electric potential in a Poisson equation:
\begin{align*}
    \rmE(\tilde{\rvx}) = \frac{1}{S_N(1)}\int\frac{\tilde{\rvx} - \tilde{\rvy}}{\|\tilde{\rvx} - \tilde{\rvy}\|^{N+1}}{p}({\rvy})\mathrm{d}\rvy \numberthis \label{eq:poisson-field}
\end{align*}
where $S_{N}(1)$ is the surface area of a unit $N$-sphere (a geometric constant), and $p(\rvy)$ is the data distribution. Samples are then generated by following the electric field lines, which are described by the ODE $\mathrm{d}{\tx} = \rmE(\tilde{\rvx})\mathrm{d}t$. In practice, 
the network is trained to estimate a normalized version of the following empirical electric field:
$
     \hat{\rmE}(\tilde{\rvx}) = c(\tilde{\rvx})\sum_{i=1}^n\frac{\tilde{\rvx}-\tilde{\rvy}_i}{\|\tilde{\rvx}-\tilde{\rvy}_i\|^{N+1}}
$,
where $c(\tilde{\rvx}) = 1/\sum_{i=1}^n\frac{1}{\|\tilde{\rvx}-\tilde{\rvy}_i\|^{N+1}}$
and $\{\tilde{\rvy}_i\}_{i=1}^n \sim \tilde{p}(\tilde{\rvy})$ is a large batch used to approximate the integral in \Eqref{eq:poisson-field}. The training objective is minimizing the $\ell_2$-loss between the neural model prediction $f_\theta(\tilde{\rvx})$ and the normalized field $\rmE(\tilde{\rvx})/\|\rmE(\tilde{\rvx})\|$ at various positions of $\tilde{\rvx}$. These positions are heuristically designed to carefully cover the regions that the sampling trajectories pass through.

\textbf{Phases of Score Field}
\citet{Xu2023StableTF} show that the score field in the forward process of diffusion models can be decomposed into three phases. When moving from the near field~(Phase 1) to the far field~(Phase 3), the perturbed data get influenced by more modes in the data distribution. They show that the posterior $p_{0|\sigma}(\rvy|\rvx) \propto p_{\sigma}(\rvx|\rvy)p(\rvy)$ serves as a phase indicator, as it gradually evolves from a delta distribution to uniform distribution when shifting from Phase 1 to Phase 3. The relevant concepts of phases have also been explored in \citet{Karras2022ElucidatingTD,choi2022perception,xiao2022tackling}. Similar to the PFGM training objective, \citet{Xu2023StableTF} approximates the score field by large batches to reduce the variance of training targets in Phase 2, where multiple data points exert comparable but distinct influences on the scores. These observations inspire us to align the phases of different $D$s in Sec~\ref{sec:diffusion}.



\section{{PFGM++: A Novel Generative Framework}}

In this section, we present our new family of generative models PFGM++, generalizing PFGM \citep{Xu2022PoissonFG} in terms of the augmented space dimensionality. We show that the electric fields in $N{+}D$-dimensional space with $D \in  \mathbb{Z}^+$ still constitute a valid generative model~(Sec~\ref{sec:pfgmpp-intro}). Furthermore, we show that the additional $D$-dimensional augmented variable can be condensed into their scalar norm due to the inherent symmetry of the electric field. To improve the training process, we propose an efficient perturbation-based objective for training PFGM++ (Sec~\ref{sec:perturbed-kernel}) without relying on the large batch approximation in the original PFGM.
\subsection{Electric field in $\bm{N{+}D}$-dimensional space}
\label{sec:pfgmpp-intro}

While PFGM~\cite{Xu2022PoissonFG} consider the electric field in a $N{+}1$-dimensional augmented space, we augment the data $\rvx$ with $D$-dimensional variables $\rvz=(z_1, \dots, z_D)$, \ie $\tilde{\rvx}=(\rvx, \rvz)$ and  $D\in \mathbb{Z}^+$. Similar to the $N{+}1$-dimensional electric field~(\Eqref{eq:poisson-field}), the electric field at the augmented data $\tilde{\mat{x}}=(\rvx, \rvz) \in \sR^{N+D}$ is:
\begin{align*}
     \mat{E}(\tilde{\mat{x}}) = \frac{1}{S_{N+D-1}(1)}\int \frac{\tilde{\mat{x}}-\tilde{\rvy}}{\|\tilde{\mat{x}}-\tilde{\rvy}\|^{N+D}}{p}({\rvy}) d {\rvy}\numberthis \label{eq:poisson-field-D}
\end{align*}
Analogous to the theoretical results presented in PFGM, with the electric field as the drift term, the ODE $d\tx {=} \mat{E}(\tilde{\mat{x}})\mathrm{d}t$ defines a surjection from a uniform distribution on an infinite \textit{$N{+}D$-dim} hemisphere and the data distribution on the \textit{$N$-dim} $\rvz{=}\bf 0$ hyperplane. However, the mapping has $SO(D)$ symmetry on the surface of $D$-dim cylinder $\sum_{i=1}^D z_i^2=r^2$ for any positive $r$.
We provide an illustrative example at the bottom of \Figref{fig:D_illustration}~($D{=}2,N{=}1$), where the electric flux emitted from a line segment~(red) has rotational symmetry through the ring area (blue) on the $z_1^2+z_2^2=r^2$ cylinder. Hence, instead of modeling the individual behavior of each $z_i$, it suffices to track the norm of augmented variables --- $r(\tilde{\rvx})=\|\rvz\|_2$ --- due to symmetry. 
Specifically, note that $\mathrm{d}z_i = \mat{E}(\tilde{\mat{x}})_{z_i}\mathrm{d}t$, and the time derivative of $r$ is 
\begin{align*}
    \frac{\mathrm{d}r}{\mathrm{d}t} &= \sum_{i=1}^D \frac{z_i}{r}\frac{\mathrm{d}z_i}{\mathrm{d}t}
    {=} \int \frac{\sum_{i=1}^Dz_i^2}{S_{N{+}D{-}1}(1)r\|\tilde{\mat{x}}-\tilde{\rvy}\|^{N{+}D}}{p}({\rvy}) \mathrm{d} {\rvy}\\
    &=\frac{1}{S_{N{+}D{-}1}(1)}\int \frac{r}{\|\tilde{\mat{x}}-\tilde{\rvy}\|^{N{+}D}}{p}({\rvy}) \mathrm{d} {\rvy}
\end{align*}
Henceforth we replace the notation for augmented data with $\tilde{\rvx}=(\rvx, r)$ for simplicity. After the symmetry reduction, the field to be modeled has a similar form as \Eqref{eq:poisson-field-D} except that the last $D$ sub-components $\{\mat{E}(\tilde{\mat{x}})_{z_i}\}_{i=1}^D$ are condensed into a scalar $E(\tilde{\mat{x}})_{r}=\frac{1}{S_{N+D-1}(1)}\int \frac{r}{\|\tilde{\mat{x}}-\tilde{\rvy}\|^{N+D}}{p}({\rvy}) \mathrm{d} {\rvy}$. Therefore, we can use the physically meaningful $r$ as the anchor variable in the ODE $\mathrm{d}\rvx/\mathrm{d}r$ by change-of-variable:
\begin{align*}
    \frac{\mathrm{d}\rvx}{\mathrm{d}r} = \frac{\mathrm{d}\rvx}{\mathrm{d}t}\frac{\mathrm{d}t}{\mathrm{d}r} = \mat{E}(\tilde{\mat{x}})_\rvx \cdot E(\tilde{\mat{x}})_{r}^{-1} = \frac{\mat{E}(\tilde{\mat{x}})_\rvx}{E(\tilde{\mat{x}})_{r}}\numberthis \label{eq:ode-pfgmpp}
\end{align*}
Indeed, the ODE $\mathrm{d}\rvx/\mathrm{d}r$ turns the aforementioned surjection into a bijection between an easy-to-sample prior distribution on the $r{=}\rmax$ hyper-cylinder~\footnote{The hyper-cylinder here is consistent with the hemisphere in PFGM~\cite{Xu2022PoissonFG}, because hyper-cylinders degrade to hyper-planes for $D=1$, which are in turn isomorphic to hemispheres.} and the data distribution on $r{=}0$~(\ie $\rvz{=}\bf 0$) hyperplane. The following theorem states the observation formally:
\begin{figure}[tbp]
\centering
\includegraphics[width=0.4\textwidth]{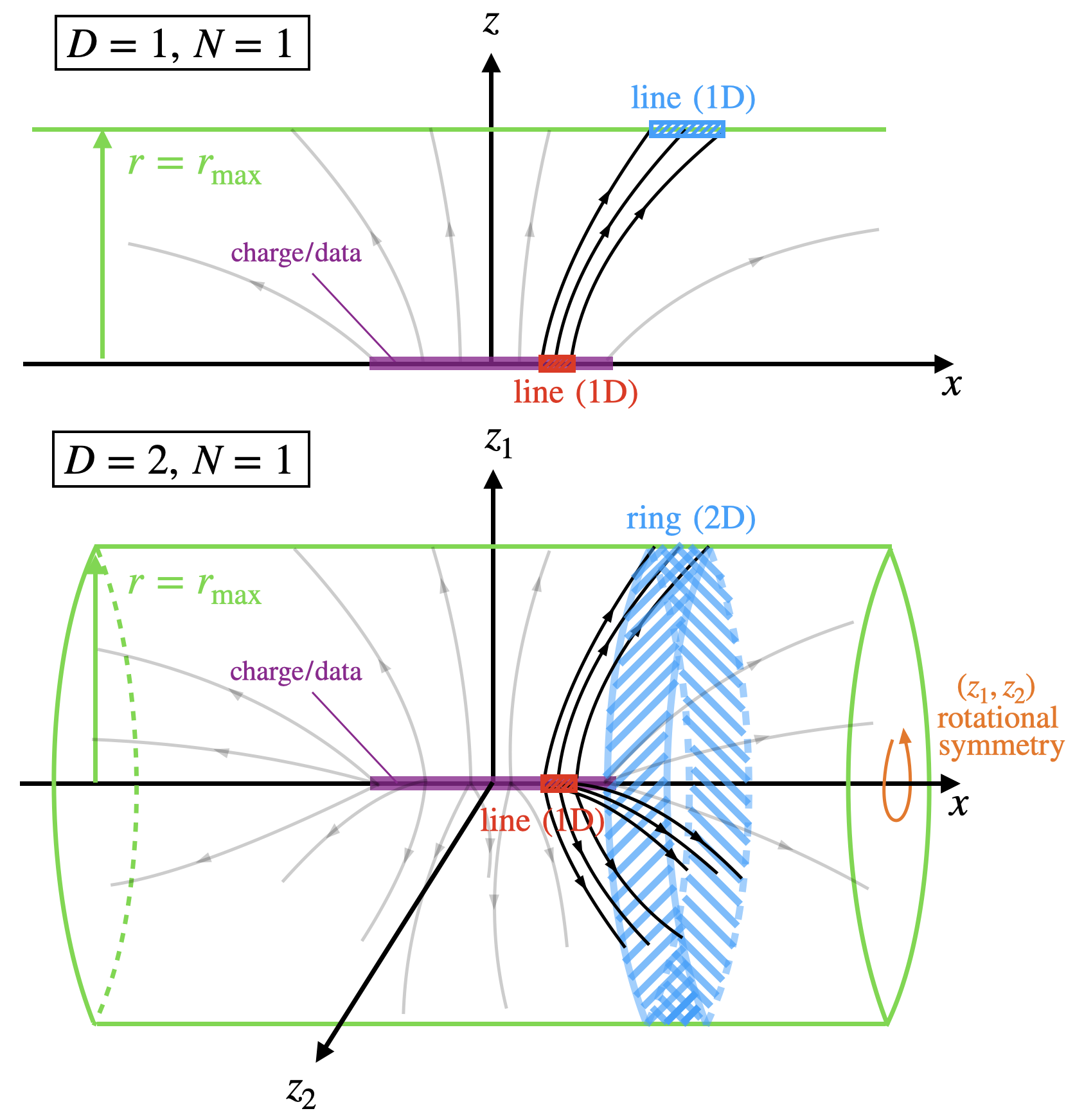}
   \vspace{-5pt}
    \caption{The augmented dimension $D$ affects electric field lines (\textcolor{gray}{gray}), which connect charge/data on a line~(\textcolor{black!30!purple}{purple}) to latent space~(\textcolor{black!30!green}{green}). When $D=1$ (top) or $D=2$ (bottom), electric field lines map the same red line segment to a blue line segment or onto a blue ring, respectively. The mapping defined by electric lines has $SO(2)$ symmetry on the surface of $z_1^2+z_2^2=r^2$ cylinder.}
    \label{fig:D_illustration}
    \vspace{-8pt}
\end{figure}

\begin{restatable}{theorem}{thmbij}
\label{thm:bijection}
Assume the data distribution $p\in \gC^1$ and $p$ has compact support. As $\rmax{\to} \infty$, for $D \in \R^+$, the ODE ${\mathrm{d}\rvx}/{\mathrm{d}r} = \mat{E}(\tilde{\mat{x}})_\rvx  /E(\tilde{\mat{x}})_{r}$ defines a bijection between $\lim_{\rmax\to \infty}p_{\rmax}(\rvx) \propto \lim_{\rmax\to \infty}{\rmax^D}/{({\|\rvx\|_2^2+ \rmax^2})^\frac{N+D}{2}}$ when $r=\rmax$ and the data distribution $p$ when $r=0$.
\end{restatable}
\begin{proofs}
    The $r$-dependent intermediate distribution of the ODE~(\Eqref{eq:ode-pfgmpp}) is $p_r(\rvx) {\propto} \int {r^D}/{\|\tilde{\mat{x}}-\tilde{\rvy}\|^{N+D}}{p}({\rvy})d\rvy$, which satisfies initial/terminal conditions, \ie $p_{r=0} {=} p$, $\lim_{\rmax\to \infty}p_{\rmax}  \propto \lim_{\rmax\to \infty}{\rmax^D}/{({\|\rvx\|_2^2+ \rmax^2})^\frac{N+D}{2}}$, as well as the continuity equation of the ODE, \ie $\partial_r p_r + \nabla_\rvx\cdot(p_r \mat{E}(\tilde{\mat{x}})_\rvx /E(\tilde{\mat{x}})_{r})=0$.
    \vspace{-10pt}
\end{proofs}
We defer the formal proof to Appendix~\ref{app:proof-bij}. Note that in the theorem we further extend the domain of $D$ from positive integers to positive real numbers. In practice, the starting condition of the ODE is some sufficiently large $\rmax$ such that $p_{\rmax}(\rvx) \appropto {\rmax^D}/{({\|\rvx\|_2^2+ \rmax^2})^\frac{N+D}{2}}$. The terminal condition is $r{=0}$, which represents the generated samples reaching the data support. The proposed PFGM++ framework thus permits choosing arbitrary $D$, including $D=1$ which recovers the original PFGM formulation. Interestingly, we will also show that when $D{\to} \infty$, PFGM++ recover the diffusion models~(Sec~\ref{sec:diffusion}). In addition, as discussed in Sec~\ref{sec:benefits}, the choice of $D$ is important, since it controls two properties of the associated electric field, \ie robustness and rigidity, which affect the sampling performance.

\subsection{New objective with Perturbation Kernel}
\label{sec:kernel}

\label{sec:perturbed-kernel}

Although the training process in PFGM can be directly applied to PFGM++, we propose a more efficient training objective to dispense with the large batch in PFGM. The objective from PFGM paper~\cite{Xu2022PoissonFG} requires sampling a large batch of data $\{\rvy_i\}_{i=1}^n {\sim} {p}^n({\rvy})$ in each training step to approximate the integral in the electric field~(\Eqref{eq:poisson-field-D}):
\begin{align*}
&\E_{\tilde{p}_{\textrm{train}}(\tx)}\E_{\{\rvy_i\}_{i=1}^n\sim p^n(\rvy)}\E_{\rvx\sim p_\sigma(\rvx|\rvy_1)}\\
&\qquad\left[\bigg\lVert f_{\theta}(\tilde{\rvx})- \frac{\sum_{i=0}^{n-1} \frac{\tilde{\mat{x}}-\Tilde{\rvy}_i}{\|\tilde{\mat{x}}-\Tilde{\rvy}_i\|^{N+D}}}{\big\Vert\sum_{i=0}^{n-1} \frac{\tilde{\mat{x}}-\Tilde{\rvy}_i}{\|\tilde{\mat{x}}-\Tilde{\rvy}_i\|^{N+D}}\big\Vert_2+\gamma}\bigg\rVert^2_2 \right]
\end{align*}
where $\tilde{p}_{\textrm{train}}$ is heuristically designed to cover the regions that the backward ODE traverses. This objective has several obvious drawbacks: \textit{(1)} The large batch incurs additional overheads; \textit{(2)} Its minimizer is a biased estimator of the electric field~(\Eqref{eq:poisson-field-D}); \textit{(3)} The large batch is incompatible with typical paired sample training of conditional generation, where each condition is paired with only one sample, such as text-to-image~\cite{Rombach2021HighResolutionIS, Saharia2022PhotorealisticTD} and text-to-3D generation~\cite{Poole2022DreamFusionTU, Nichol2022PointEAS}.

To remedy these issues, we propose a perturbation-based objective without the need for the large batch, while achieving an unbiased minimizer and enabling paired sample training of conditional generation. Inspired by denoising score-matching~\cite{Vincent2011ACB}, we design the perturbation kernel to guarantee that the minimizer in the following square loss objective matches the ground-truth electric field in \Eqref{eq:poisson-field-D}:
\begin{align*}
    \E_{r\sim p(r)}\E_{p({\rvy})}\E_{p_{r}(\rvx|{\rvy})}\left[\| f_{\theta}(\tilde{\rvx})- ({{\tilde{\rvx}-\tilde{\rvy}}})\|_2^2\right] \numberthis \label{eq:new-D-aug}
\end{align*}
where $r\in(0, \infty)$, $p(r)$ is the training distribution over $r$, $p_{r}(\rvx|{\rvy})$ is the perturbation kernel and $\ty {=} (\rvy,0)$/$\tx{=}(\rvx, r)$ are the clean/perturbed augmented data. The minimizer of \Eqref{eq:new-D-aug} is $f^*_{\theta}(\tilde{\rvx}) {\propto} {\int p_r(\rvx|{{\rvy}})({\tilde{\rvx}-\tilde{\rvy}})p({\rvy}) \mathrm{d}{\rvy}}$, which matches the direction of electric field $\mat{E}(\tilde{\mat{x}}) {\propto} \int {(\tilde{\mat{x}}-\tilde{\rvy})}/{\|\tilde{\mat{x}}-\tilde{\rvy}\|^{N+D}}{p}({\rvy}) \mathrm{d} {\rvy}$ when setting the perturbation kernel to $p_{r}(\rvx|{\rvy}) {\propto}   {1}/{({\|\rvx-{\rvy}\|_2^2+ r^2})^\frac{N+D}{2}}$. Denoting the $r$-dependent intermediate marginal distribution as $p_r(\rvx) {=} \int p_r(\rvx|\rvy)p(\rvy)d\rvy$, the following proposition states that the choice of $p_r(\cdot|\rvy)$ guarantee that the minimizer of the square loss to match the direction of the electric field:
\begin{restatable}{proposition}{thmmin}
\label{thm:minimizer}
    With perturbation kernel $p_{r}(\rvx|{\rvy}) \propto {1}/{({\|\rvx-{\rvy}\|_2^2+ r^2})^\frac{N+D}{2}}$, for $\forall \rvx\in \R^N, r>0$, the minimizer $f^*_\theta(\tx)$ in the PFGM++ objective~(\Eqref{eq:new-D-aug}) matches the direction of electric field $\mat{E}(\tx)$ in \Eqref{eq:poisson-field-D}. Specifically, 
$
        f^*_\theta(\tx) = ({S_{N+D-1}(1)}/{p_r(\rvx)})\mat{E}(\tx)
$.
\end{restatable}
We defer the proof to Appendix~\ref{app:thm-minimizer}. The proposition indicates that the minimizer $f^*_\theta(\tx)$ can match the direction of $\mat{E}(\tx)$ with sufficient data and model capacity. The current training target in \Eqref{eq:new-D-aug} is the directional vector between the clean data $\ty$ and perturbed data $\tx$ akin to denoising score-matching for diffusion models~\cite{Song2021ScoreBasedGM, Karras2022ElucidatingTD}. In addition, the new objective allows for conditional generations under a one-sample-per-condition setup. Since the perturbation kernel is isotropic, we can decompose $p_r(\cdot|\rvy)$ in hyperspherical coordinates to $\gU_{\psi}(\psi)p_r(R)$, where $\gU_{\psi}$ is the uniform distribution over the angle component and the distribution of the perturbed radius $R=\|\rvx - \rvy\|_2$ is
\begin{align*}
    p_r(R) \propto \frac{R^{N-1}}{({R^2+ r^2})^\frac{N+D}{2}}
\end{align*}
We defer the practical sampling procedure of the perturbation kernel to Appendix~\ref{app:sample-prior}. The mean of the $r$-dependent radius distribution $p_r(R)$ is around $r\sqrt{N/D}$. Hence we explicitly normalize the target in \Eqref{eq:new-D-aug} by $r/\sqrt{D}$, to keep the norm of the target around the constant $\sqrt{N}$, similar to diffusion models~\cite{Song2021ScoreBasedGM}. In addition, we drop the last dimension of the target because it is a constant ---  $({\tilde{\rvx}-\tilde{\rvy}})_{r}/({r}/{\sqrt{D}})=\sqrt{D}$. Together, the new objective is
\begin{align*}
    \E_{r\sim p(r)}\E_{p(\tilde{\rvy})}\E_{p_{r}(\tx|\tilde{\rvy})}\left[\Big\lVert f_{\theta}(\tilde{\rvx})- \frac{{\rvx}-{\rvy}}{{r}/{\sqrt{D}}}\Big\rVert_2^2\right] \numberthis \label{eq:pfgmpp-obj}
\end{align*}
After training the neural network through objective \Eqref{eq:pfgmpp-obj}, we can use the ODE~(\Eqref{eq:ode-pfgmpp}) anchored by $r$ to generate samples, \ie $\mathrm{d}\rvx /\mathrm{d}r = \mat{E}(\tilde{\mat{x}})_\rvx/E(\tilde{\mat{x}})_{r}= f_\theta(\tx) / \sqrt{D}$, starting from the prior distribution $p_{\rmax}$.

\section{Diffusion Models as $\bm{D{\to} \infty}$ Special Cases}
\label{sec:diffusion}

Diffusion models generate samples by simulating ODE/SDE involving the score function $\nabla_\rvx \log p_{\sigma}(\rvx)$ at different intermediate distributions $p_{\sigma}$~\cite{Song2021ScoreBasedGM, Karras2022ElucidatingTD}, where $\sigma$ is the standard deviation of the Gaussian kernel. In this section, we show that both sampling and training schemes in diffusion models are equivalent to those in $D{\to} \infty$ case under the PFGM++ framework. To begin with, we show that the electric field~(\Eqref{eq:poisson-field-D}) in PFGM++ has the same direction as the score function when $D$ tends to infinity, and their sampling processes are also identical.
\begin{restatable}{theorem}{thmfield}
\label{thm:inf-D-field}
Assume the data distribution $p\in \gC^1$. Consider taking the limit $D\to\infty$ while holding $\sigma = r/\sqrt{D}$ fixed. Then, for all $\rvx$, 
\begin{align*}
    \lim_{\substack{D\to \infty\\r=\sigma\sqrt{D}}}\bigg\lVert \frac{\sqrt{D}}{E(\tilde{\mat{x}})_{r}}\mat{E}(\tilde{\mat{x}})_{\rvx} - \sigma\nabla_\rvx\log p_{\sigma=r/\sqrt{D}}(\rvx)\bigg\rVert_2=0
\end{align*}
where $\mat{E}(\tx=(\rvx, r))_{\rvx}$ is given in \Eqref{eq:poisson-field-D}. Further, given the same initial point, the trajectory of the PFGM++ ODE~($\mathrm{d}\rvx/\mathrm{d}r{=} \mat{E}(\tilde{\mat{x}})_{\rvx}/ E(\tilde{\mat{x}})_{r}$) matches the diffusion ODE~\cite{Karras2022ElucidatingTD} ($\mathrm{d}\rvx/\mathrm{d}t {=} -\dot{\sigma}(t)\sigma(t)\nabla_\rvx \log p_{\sigma(t)}(\rvx)$) in the same limit. \vspace{-5pt}
\end{restatable}
\begin{proofs} By re-expressing the $\rvx$ component $\mat{E}(\tilde{\mat{x}})_{\rvx}$ in the electric field and the score $\nabla_\rvx \log p_{\sigma}$ in diffusion models, the proof boils down to show that $\lim_{{D\to \infty,r=\sigma\sqrt{D}}}p_r(\rvx|\rvy)\propto \exp(-{\|\rvx-\rvy\|_2^2}/{ 2\sigma^2})$ for $\forall \rvx, \rvy \in \R^{N+D}$:
\begin{align*}
    &\lim_{{D\to \infty,r=\sigma\sqrt{D}}}\frac{1}{({\|\rvx-{\rvy}\|_2^2+r^2)^\frac{N+D}{2}}} 
     \\ \propto &\lim_{{D\to \infty,r=\sigma\sqrt{D}}}e^{- \frac{(N+D)}{2}{\rm ln}(1+\frac{\|\rvx-\rvy\|^2}{r^2})}
     \\ = &\lim_{{D\to \infty,r=\sigma\sqrt{D}}}e^{- \frac{(N+D)\|\rvx-\rvy\|_2^2}{ 2r^2} } \numberthis  \label{eq:app-exp} 
     =e^{-\frac{\|\rvx-\rvy\|_2^2}{ 2\sigma^2}} 
\end{align*}    
The equivalence of trajectories can be proven by change-of-variable $\mathrm{d}\sigma = \mathrm{d}r /\sqrt{D}$. Their prior distributions are also the same since $\lim_{D \to \infty} p_{\rmax=\sigma_{\textrm{max}}\sqrt{D}} (\rvx)=\gN(\mathbf{0}, \sigma_{\textrm{max}}\mI)$. 
\end{proofs}
We defer the formal proof to Appendix~\ref{app:proof-thmfield}. Since $\|\rvx-\rvy\|_2^2/r^2\approx {N/D}$ when $\rvx \sim p_r(\rvx), \rvy \sim p(\rvy)$, \Eqref{eq:app-exp} approximately holds under the condition $D\gg N$. Remarkably, the theorem states that PFGM++ recover the field and sampling of previous popular diffusion models, such as VE/VP~\cite{Song2020ImprovedTF} and EDM~\cite{Karras2022ElucidatingTD}, by choosing the appropriate schedule and scale function in \citet{Karras2022ElucidatingTD}.

In addition to the field and sampling equivalence, we demonstrate that the proposed PFGM++ objective~(\Eqref{eq:pfgmpp-obj})
with perturbation kernel $p_r(\rvx|\rvy) \propto   {1}/{({\|\rvx-{\rvy}\|_2^2+ r^2})^\frac{N+D}{2}}$ recovers the weighted sum of the denoising score matching objective~\cite{Vincent2011ACB} for training continuous diffusion model~\cite{Karras2022ElucidatingTD, Song2021ScoreBasedGM} when $D {\to} \infty$. All previous objectives for training diffusion models can be subsumed in the following form~\cite{Karras2022ElucidatingTD}, under different parameterizations of the neural networks $f_\theta$:
\begin{align*}
    \E_{\sigma \sim p(\sigma)}\lambda(\sigma)\E_{p({\rvy})}\E_{p_{\sigma}({\rvx}|{\rvy})}\left[\Big\lVert f_{\theta}({\rvx}, \sigma)- \frac{{\rvx}-{\rvy}}{\sigma}\Big\rVert_2^2\right] \numberthis \label{eq:diffusion-obj}
\end{align*}
where
$
p_\sigma({\rvx}|\rvy) \propto \exp({-{\|\rvx-\rvy\|_2^2}/{2\sigma^2}})
$. The objective of the diffusion models resembles the one of PFGM++~(\Eqref{eq:pfgmpp-obj}). Indeed, we show that when $D{\to} \infty$, the minimizer of the proposed PFGM++ objective at $\tx{=}(\rvx,r)$ is $f^*_{\theta}(\rvx, r=\sigma\sqrt{D}) {=}\sigma \nabla_\rvx \log p_{\sigma}(\rvx)$, the same as the minimizer of diffusion objective at the noise level $\sigma{=}r/\sqrt{D}$.

\begin{restatable}{proposition}{propobj}
\label{prop:obj}
When $r=\sigma\sqrt{D}, D\to \infty$, the minimizer in the PFGM++ objective~(\Eqref{eq:pfgmpp-obj}) is \textbf{equaivalent to} the minimizer in the weighted sum of denoising score matching objective~(\Eqref{eq:diffusion-obj})
\end{restatable}
We defer the proof to Appendix~\ref{app:proof-propobj}. The proposition states that the training objective of diffusion models is essentially the same as PFGM++'s when $D{\to} \infty$. Combined with Theorem~\ref{thm:inf-D-field}, PFGM++ thus recover both the training and sampling processes of diffusion models when $D{\to} \infty$.

\paragraph{Transfer hyperparameters to finite $\bm{D}$s} 
The training hyperparameters of diffusion models~($D{\to }\infty$) have been highly optimized through a series of works~\cite{Ho2020DenoisingDP, Song2021ScoreBasedGM,  Karras2022ElucidatingTD}. It motivates us to transfer hyperparameters, such as $\rmax$ and $p(r)$, of $D{\to}\infty$ to finite $D$s. Here we present an alignment method that enables a ``zero-shot" transfer of hyperparameters across different $D$s. Our alignment method is inspired by the concept of phases in \citet{Xu2023StableTF}, which is related to the variation of training targets. We aim to align the intermediate marginal distributions $p_r$ for two distinct $D_1, D_2>0$. In Appendix~\ref{app:phase-align}, we demonstrate that when $r \propto \sqrt{D}$, the phase of the intermediate distribution $p_r$ is approximately invariant to all $D>0$~(including $D{\to}\infty$). In other words, when ${r_{D_1}}/{r_{D_2}}=\sqrt{D_1/D_2}$, the phases of $p_{r_{D_1}}$ and $p_{r_{D_2}}$, under $D_1$ and $D_2$ respectively, are roughly aligned. Theorem~\ref{thm:inf-D-field} further shows that the relation $r{=}\sigma\sqrt{D}$ makes PFGM++ equivalent to  diffusion models when $D{\to} \infty$. Together, the $r{=}\sigma\sqrt{D}$ formula aligns the phases of $p_\sigma$ in diffusion models and $p_{r=\sigma\sqrt{D}}$ in PFGM++ for $\forall D{>}0$. 
Such alignment enables directly transferring the finely tuned hyperparameters $\sigma_{\textrm{max}}, p(\sigma)$ in previous state-of-the-art diffusion models~\cite{Karras2022ElucidatingTD} with $\rmax {=} \sigma_{\textrm{max}}\sqrt{D}, p(r) {=} p(\sigma {=} r/\sqrt{D})/\sqrt{D}$. \textit{We put the practical hyperparameter transfer procedures in Appendix~\ref{app:transfer-diff}.}

We empirically verify the alignment formula on the CIFAR-10~\cite{Krizhevsky2009LearningML}. \citet{Xu2023StableTF} shows that the posterior $p_{0|r}(\rvy|\rvx)\propto p_r(\rvx|\rvy)p(\rvy)$ gradually grows towards a uniform distribution from the near to the far field. As a result, the mean total variational distance~(TVD) between a uniform distribution and the posterior serves as an indicator of the phase of $p_r$:
$
    \E_{p_{r}(\rvx)}\textrm{TVD}\left(U(\cdot)\parallel p_{0|r}(\cdot|\rvx)\right)
$.
\Figref{fig:align} reports the mean TVD before and after the $r{=}\sigma\sqrt{D}$ alignment. We observe that the mean TVDs of a wide range of $D$s take similar values after the alignment, suggesting that the phases of $p_{r=\sigma\sqrt{D}}$ are roughly aligned for different $D$s.

\begin{figure}[t]
\centering
\vspace{-10pt}
\subfigure[No alignment]{\includegraphics[width=0.248\textwidth]{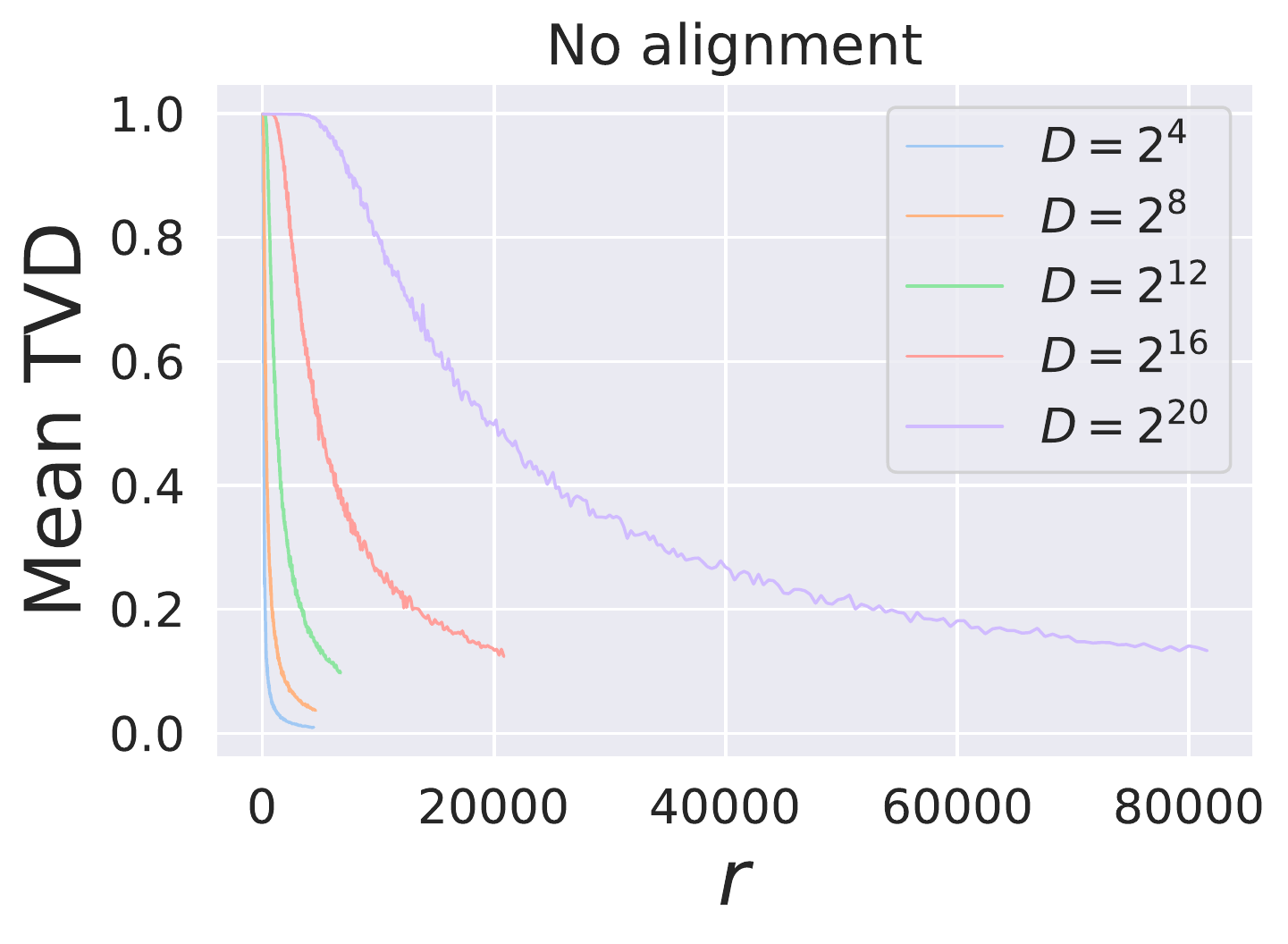}}\hfill
\subfigure[$r=\sigma\sqrt{D}$ alignment]{\includegraphics[width=0.23\textwidth]{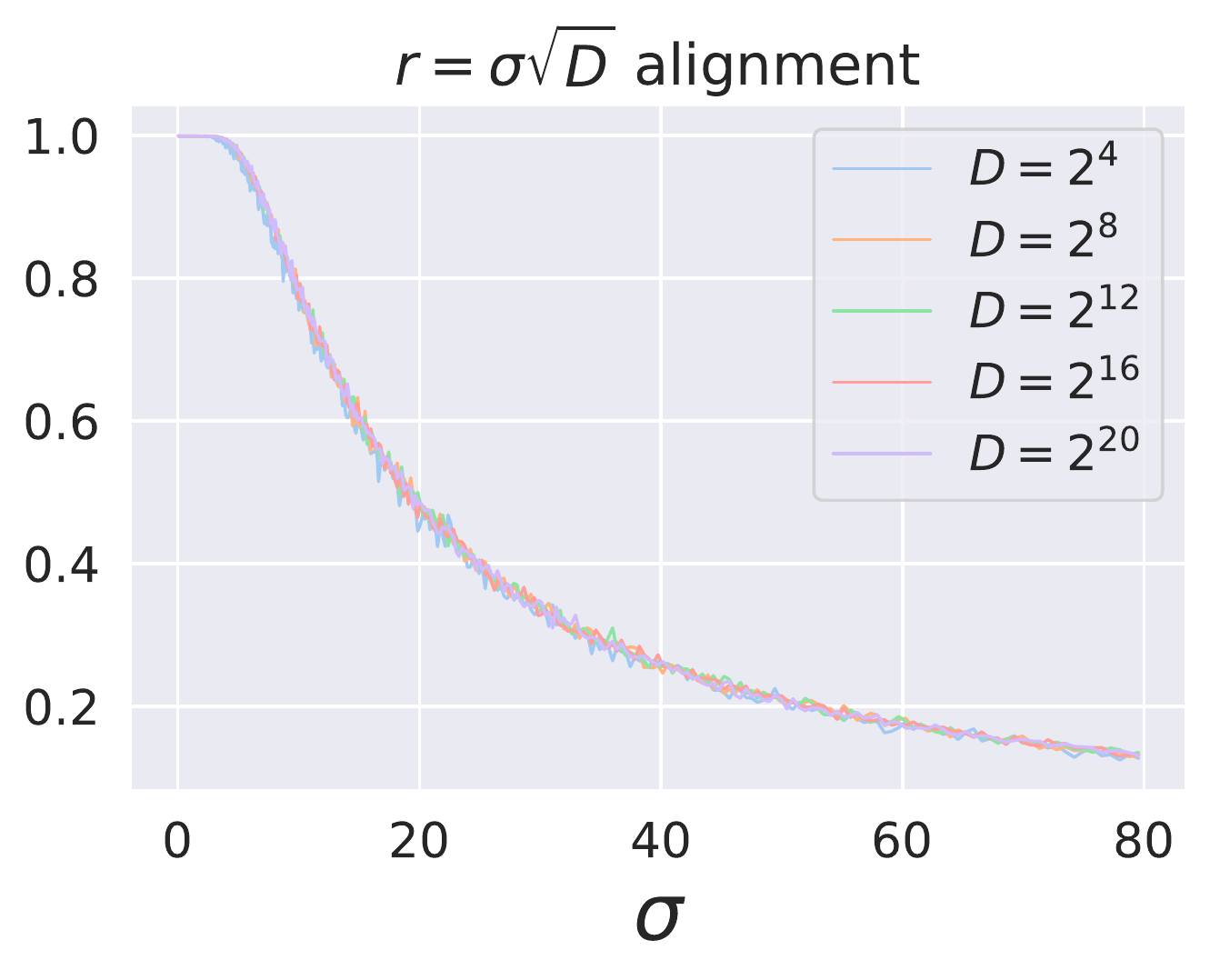}}
\vspace{-10pt}
    \caption{Mean TVD between the posterior $p_{0|r}(\cdot|\rvx)$ ($\rvx$ is perturbed sample) and the uniform prior, w/o~\textbf{(a)} and w/~\textbf{(b)} the phase alignment~($r=\sigma\sqrt{D}$).}
    \label{fig:align}
    \vspace{-20pt}
\end{figure}

\section{Balancing Robustness and Rigidity}

\label{sec:benefits}
In this section, we first delve into the behaviors of PFGM++ with different $D$s~(Sec~\ref{sec:behavior}) based on the alignment formula. Then we demonstrate how to leverage $D$ to balance the robustness and rigidity of models~(Sec~\ref{sec:balance}). We defer all experimental details in this section to Appendix~\ref{app:be-exp}.

\subsection{Behavior of perturbation kernel when varying $\bm{D}$}
\label{sec:behavior}

\begin{figure*}[htbp]
\centering
  \subfigure[]{\includegraphics[width=0.315\textwidth]{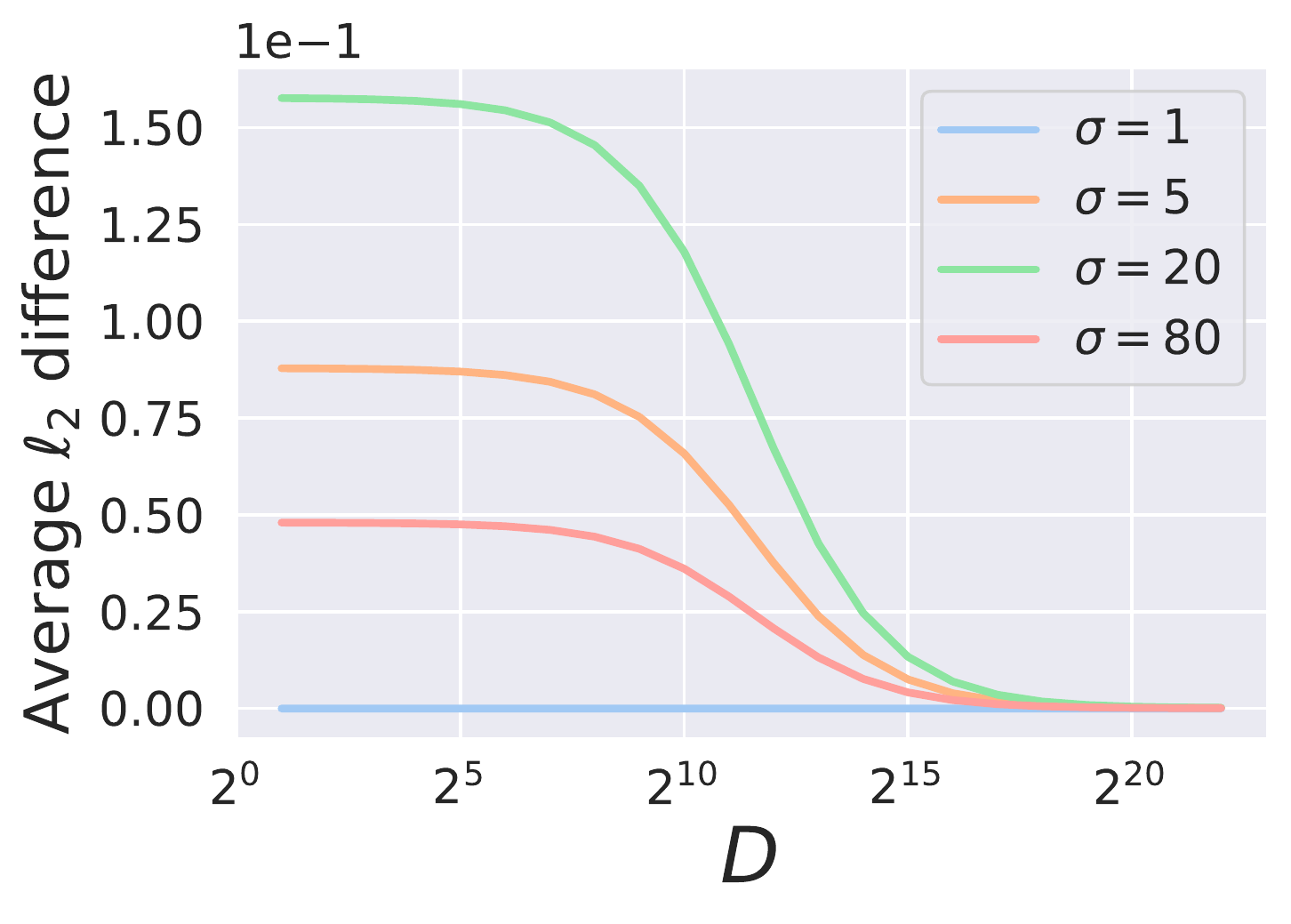}\label{fig:dis}}\hfill
   \subfigure[]{ \includegraphics[width=0.31\textwidth]{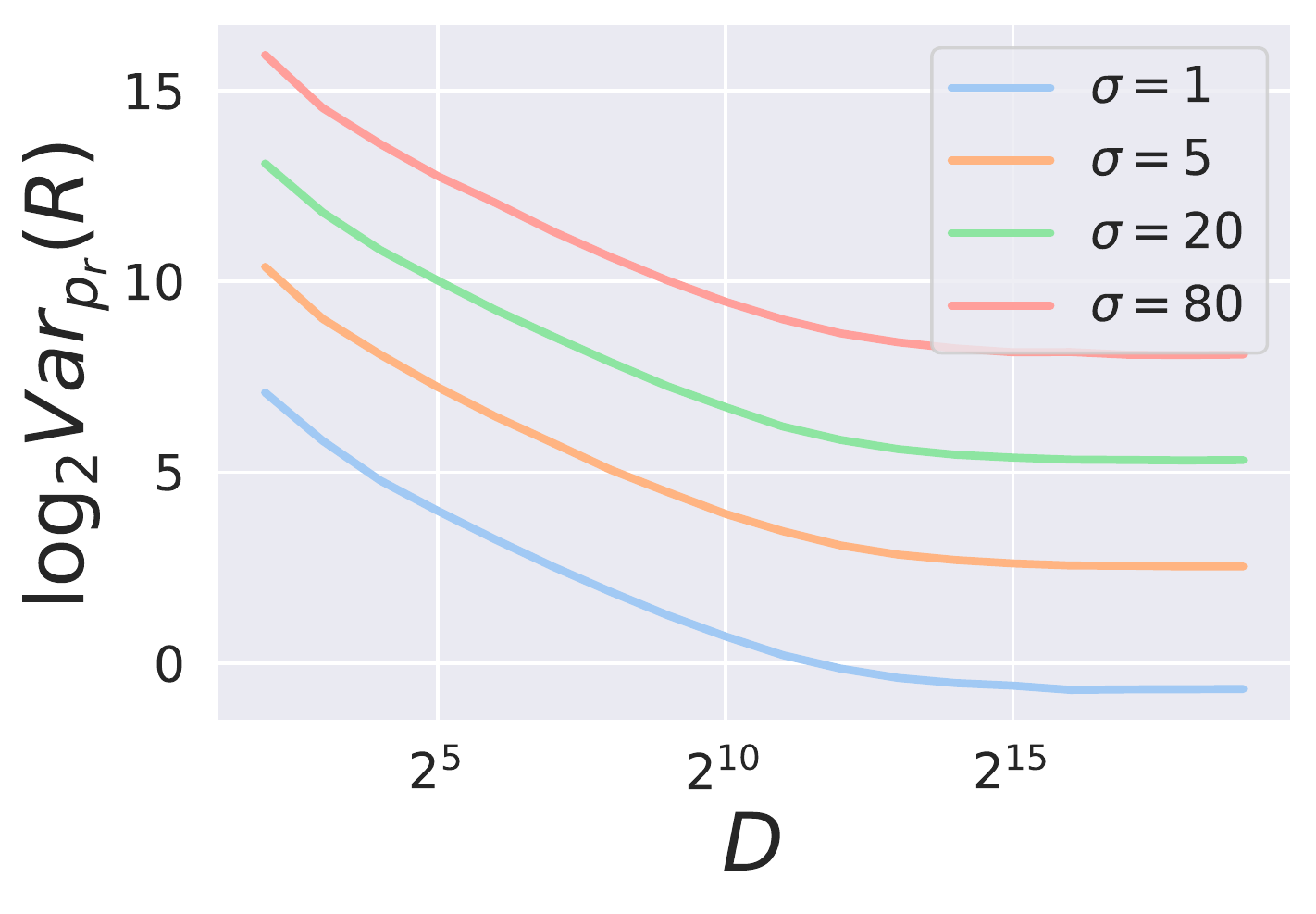}  \label{fig:var}}\hfill
  \subfigure[]{ \includegraphics[ width=0.28\textwidth, trim = 0cm 1cm 0cm 1.5cm]{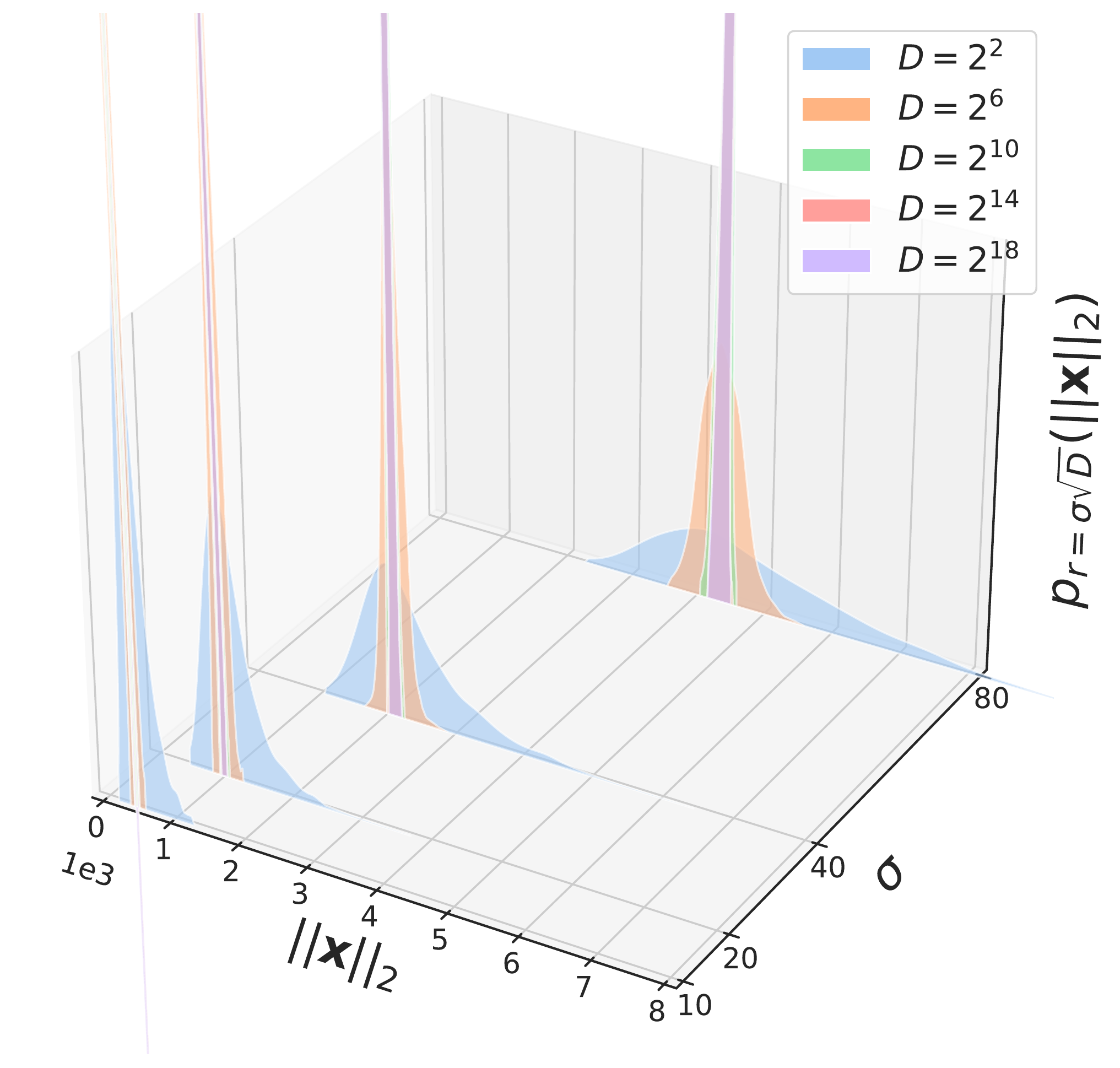}  \label{fig:norm-r}}
  \vspace{-6pt}
  \caption{\textbf{(a)} Average $\ell_2$ difference between scaled electric field and score function, versus $D$. \textbf{(b)} Log-variance of radius distribution versus $D$. \textbf{(c)} Density of radius distributions $p_{r=\sigma\sqrt{D}}(R)$ with varying $\sigma$ and $D$. }
    \vspace{-6pt}
\end{figure*}

According to Theorem~\ref{thm:inf-D-field}, when $D{\to} \infty$, the field in PFGM++ has the same direction as the score function, \ie $ \sqrt{D}\mat{E}(\tilde{\mat{x}})_{\rvx}/E(\tilde{\mat{x}})_{r} {=} \sigma \nabla_\rvx\log p_{\sigma=r/\sqrt{D}}(\rvx)$. In addition to the theoretical analysis, we provide further empirical study to characterize the convergence towards diffusion models as $D$ increases. \Figref{fig:dis} reports the average $\ell_2$ difference between the two quantities, \ie $
    \E_{p_{\sigma}(\rvx)}[\|\sqrt{D}\mat{E}(\tilde{\mat{x}})_{\rvx}/E(\tilde{\mat{x}})_{r} {-} \sigma \nabla_\rvx\log p_{\sigma}(\rvx)\|_2]
$ with $r{=}\sigma\sqrt{D}$. We observe that the difference monotonically decreases as a function of $D$, and converges to $0$ as predicted by theory. For $\sigma{=}1$, the distance remains $0$ since the empirical posterior $p_{0|r}$ concentrates around a single example for all $D$.

Next, we examine the behavior of the perturbation kernel after the phase alignment. Recall that the isotropic perturbation kernel $p_r(\rvx|\rvy) \propto {1}/{({\|\rvx-{\rvy}\|_2^2+ r^2})^\frac{N+D}{2}}$ can be decomposed into a uniform angle component and a radius distribution $p_r(R) \propto {R^{N-1}}/{({R^2+ r^2})^\frac{N+D}{2}}$. \Figref{fig:var} shows the variance of the radius distribution significantly decreases as $D$ increases. The results imply that with relatively large $r$, the norm of the training sample in $p_r(\rvx)$ becomes increasingly concentrated around a specific value as $D$ increases, reaching its highest level of concentration as $D{\to}\infty$~(diffusion models). \Figref{fig:norm-r} further shows the density of training sample norms in $p_{r=\sigma\sqrt{D}}(\rvx)$ on CIFAR-10. We can see that the range of the high-mass region gradually shrinks when $D$ increases.

\subsection{Balancing the trade-off by controlling $\bm{D}$}
\label{sec:balance}

As noted in \citet{Xu2022PoissonFG}, diffusion models~($D{\to}\infty$) are more susceptible to estimation errors compared to PFGM~($D{=}1$) due to the strong correlation between $\sigma$ and the training sample norm, as demonstrated in \Figref{fig:norm-r}. When $D$ and $r$ are large, the marginal distribution $p_r(\rvx)$ is approximately supported on the sphere with radius $r\sqrt{N/D}$. The backward ODE can lead to unexpected results if the sampling trajectories deviate from this  norm-$r$  relation present in training samples. This phenomenon was empirically confirmed by \citet{Xu2022PoissonFG} for PFGM/diffusion models~($D{=}1$ and $D{\to} \infty$ cases) using a weaker architecture NCSNv2~\cite{Song2020ImprovedTF}, where PFGM was shown to be significantly more robust than diffusion models.

Smaller $D$, however, implies a heavy-tailed input distribution. \Figref{fig:norm-r} illustrates that the examples used as the input to the neural network have a broader range of norms when $D$ is small. In particular, when $D{<}2^5$, the variance of perturbation radius can be larger than $2^{10}$~(\Figref{fig:var}). This broader input range can be challenging for any finite-capacity neural network. Although \citet{Xu2022PoissonFG} introduced heuristics to bypass this issue in the $D{=}1$ case, \eg restricting the sampling/training regions, these heuristics also prevent the sampling process from faithfully recovering the data distribution.

Thus, we can view $D$ as a parameter to optimize so as to balance the robustness of generation against rigidity that helps learning. 
Increased robustness allows practitioners to use smaller neural networks, e.g., by applying post-training quantization~\cite{Han2015DeepCC, Banner2018PostT4}. 
In other words, smaller $D$ allows for more aggressive quantization/larger sampling step sizes/smaller architectures.
These can be crucial in real-world applications where computational resources and storage are limited. On the other hand, such gains need to be balanced against easier training afforded by larger values of $D$. The ability to optimize the balance by varying $D$ can be therefore advantageous. We expect that there exists a sweet spot of ${D}$ in the middle striking the balance, as the model robustness and rigidity go in opposite directions.

\vspace{-5pt}
\section{Experiments}

\label{sec:experiment}
In this section, we assess the performance of different generative models on image generation tasks~(Sec~\ref{sec:image}), where models with some median $D$s outperform previous state-of-the-art diffusion models~($D{\to} \infty$), consistent with the sweet spot argument in Sec~\ref{sec:benefits}. We also demonstrate the improved robustness against three kinds of error as $D$ decreases~(Sec~\ref{sec:robust}).

\subsection{Image generation}
\label{sec:image}

We consider the widely used benchmarks CIFAR-10 $32{\times} 32$~\cite{Krizhevsky2009LearningML} and FFHQ $64{\times} 64$~\cite{Karras2018ASG} for image generation. For training, we utilize the improved NCSN++/DDPM++ architectures, preconditioning techniques and hyperparameters from the state-of-the-art diffusion model  EDM~\cite{Karras2022ElucidatingTD}. Specifically, we use the alignment method developed in Sec~\ref{sec:diffusion} to transfer their tuned critical hyperparameters $\sigma_{\textrm{max}},\sigma_{\textrm{min}}, p(\sigma)$ in the $D{\to} \infty$ case to finite $D$ cases. According to the experimental results in \citet{Karras2018ASG}, the log-normal training distribution $p(\sigma)$ has the most substantial impact on the final performances. For ODE solver during sampling, we use Heun's $2^{\textrm{nd}}$ method~\cite{Ascher1998ComputerMF} as in EDM.


\begin{table}[h]
    \scriptsize
    \centering
    \caption{CIFAR-10 sample quality~(FID) and number of function evaluations~(NFE).}
    \begin{tabular}{l c c c c}
    \toprule
         &{Min FID $\downarrow$}  & {Top-3 Avg FID $\downarrow$} & {NFE $\downarrow$}\\
         \midrule
DDPM~\citep{Ho2020DenoisingDP}& $3.17$& -&$1000$\\
DDIM~\cite{Song2021DenoisingDI}  & $4.67$ &-& $50$ \\
          VE-ODE~\citep{Song2021ScoreBasedGM}&  $5.29$ &-& $194$\\
          VP-ODE~\citep{Song2021ScoreBasedGM} & $2.86$&-& $134$\\
         PFGM~\cite{Xu2022PoissonFG} & ${2.48}$ & -&${104}$\\
        \midrule
        \textbf{\textit{PFGM++ (unconditional)}} \\
        \midrule
        $D=64$  & $1.96$& $1.98$ &$35$\\
        $D=128$  & $1.92$& $1.94$&$35$\\
        $D=2048$  & $\bf 1.91$& $\bf 1.93$& $35$\\ 
        $D=3072000$ & $1.99$& $2.02$ &$35$\\
    $D\to \infty$~\cite{Karras2022ElucidatingTD}  &$1.98$& $2.00$  & $35$\\
        \midrule
        \textbf{\textit{PFGM++ (class-conditional)}} \\
        \midrule
        $D=2048$  & $\bf 1.74$& -& $35$\\ 
    $D\to \infty$~\cite{Karras2022ElucidatingTD}  &$1.79$& -  & $35$\\
        \bottomrule
    \end{tabular}
    \label{tab:cifar}
\end{table}

\begin{table}[h]
    \small
    \centering
    \caption{FFHQ sample quality~(FID) with 79 NFE in unconditional setting}
    \begin{tabular}{l c c c}
    \toprule
         &\scalebox{0.8}{Min FID $\downarrow$}  &\scalebox{0.8}{Top-3 Avg FID} $\downarrow$\\
        \midrule
        $D=128$  &$\bf 2.43$& $2.48$\\
        $D=2048$  & $2.46$& $\bf 2.47$\\ 
        $D=3072000$ & $2.49$& $2.52$\\
    $D\to \infty$~\cite{Karras2022ElucidatingTD} &$2.53$ &$2.54$ \\
        \bottomrule
    \end{tabular}
    \label{tab:ffhq}
\end{table}

We compare models trained with $D{\to} \infty$~(EDM) and $D {\in} \{64, 128, 2048, 3072000 \}$. In our experiments, we exclude the case of $D{=}1$~(PFGM) because the perturbation kernel is extremely heavy-tailed~(\Figref{fig:var}), making it difficult to integrate with our perturbation-based objective without the restrictive region heuristics proposed in \citet{Xu2022PoissonFG}. We also exclude the small $D=64$ for the higher-resolution dataset FFHQ. We include several popular generative models for reference and defer more training and sampling details to Appendix~\ref{app:exp}.

\textbf{Results:}  In Table~\ref{tab:cifar} and Table~\ref{tab:ffhq}, we report the sample quality measured by the FID score~\cite{Heusel2017GANsTB}~(lower is better), and inference speed measured by the number of function evaluations. As in EDM, we report the minimum FID score over checkpoints. Since we empirically observe a large variation of FID scores on FFHQ across checkpoints~(Appendix~\ref{app:eval}), we also use the average FID score over the Top-3 checkpoints as another metric. Our main findings are \textbf{(1) Median $\bm{D}$s outperform diffusion models~($\bm{D{\to} \infty}$) under PFGM++ framework.} We observe that the $D{=}2048/128$ cases achieve the best performance among our choices on CIFAR-10 and FFHQ, with a current state-of-the-art min FID score of $1.91/2.43$ in unconditional setting, using the perturbation-based objective. In addition, the $D{=}2048$ case obtain better Top-3 average FID scores~($1.93/2.47$) than EDM~($2.00/2.54$) on both datasets in unconditional setting, and achieve current state-of-the-art FID score of $1.74$ on CIFAR-10 class-conditional setting. \textbf{(2) There is a sweet spot between $\bm{(1,\infty)}$.} Neither small $D$ nor infinite $D$ obtains the best performance, which confirms that there is a sweet spot in the middle, well-balancing rigidity and robustness. \textbf{(3) Model with $\bm{D{\gg} N}$  recovers diffusion models.} We find that model with sufficiently large $D$ roughly matches the performance of diffusion models, as predicted by the theory. Further results in Appendix~\ref{app:stf} show that $D{=}3072000$ and diffusion models obtain the same FID score when using a more stable training target~\cite{Xu2023StableTF} to mitigate the variations between different runs and checkpoints.

\subsection{Model robustness versus $\bm{D}$}
\begin{figure}[htbp]
    \centering
\subfigure{\includegraphics[width=0.251\textwidth]{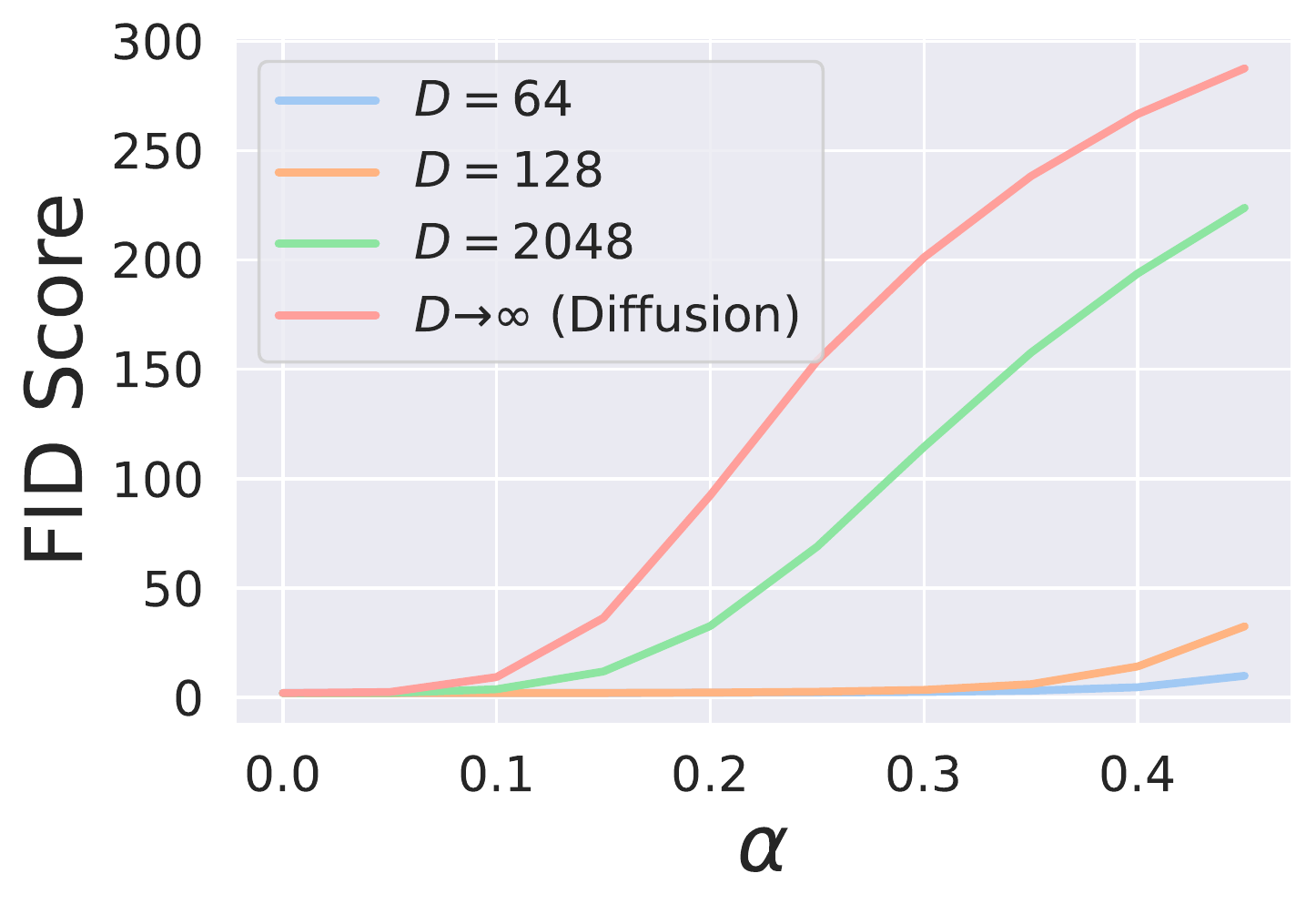}\label{fig:alpha}}\hfill
\subfigure{\includegraphics[width=0.23\textwidth]{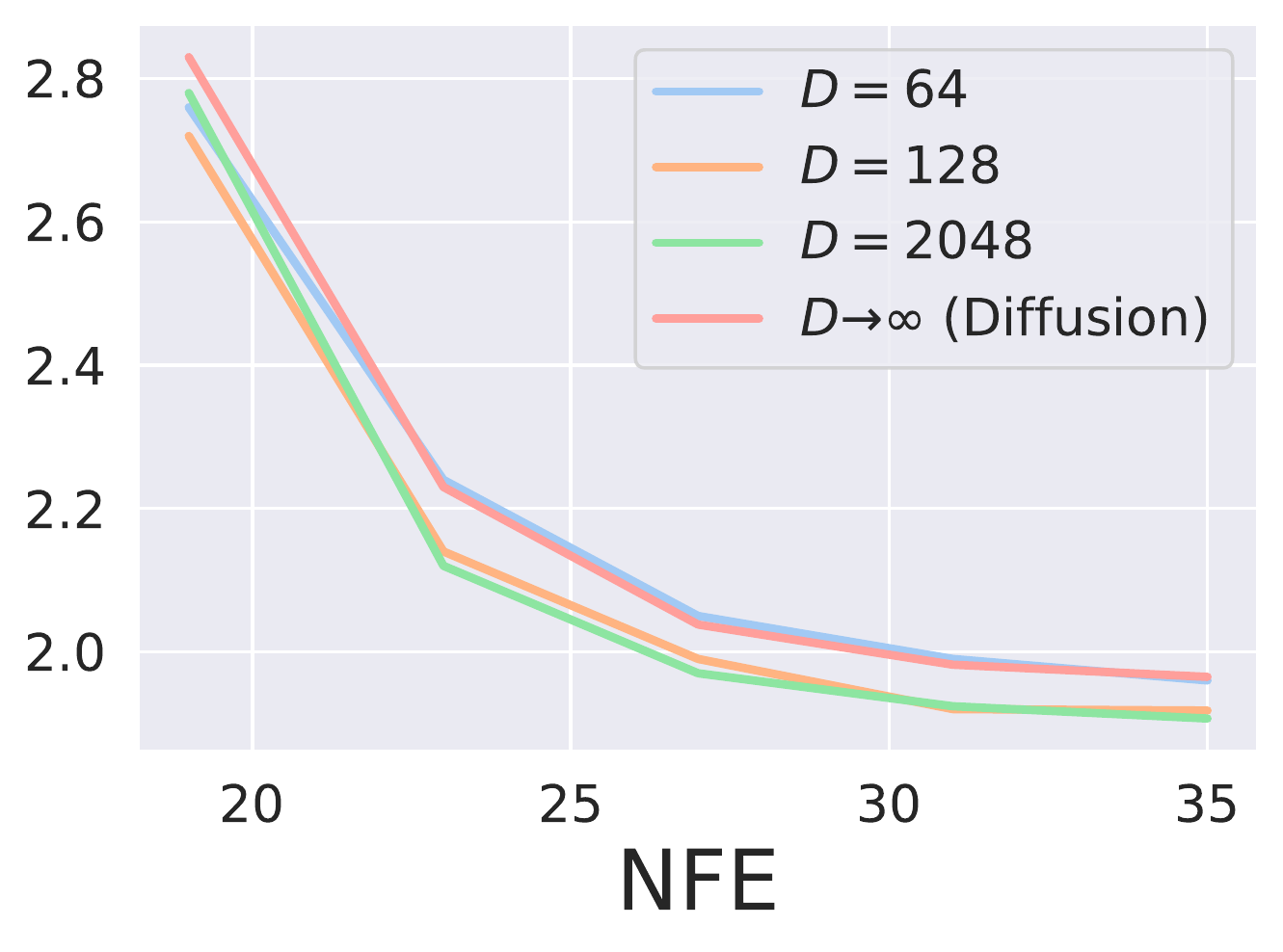}\label{fig:nfe}}
\centering
\vspace{-16pt}
\caption{FID score versus \textbf{(left)}~$\alpha$ and \textbf{(right)}~NFE on CIFAR-10.}
\vspace{-3pt}
\end{figure}
\label{sec:robust}

In Section~\ref{sec:benefits}, we show that the model robustness degrades with an increasing $D$ by analyzing the behavior of perturbation kernels. To further validate the phenomenon, we conduct three sets of experiments with different sources of errors on CIFAR-10. We defer more details to Appedix~\ref{app:robust-exp}. Firstly, we perform controlled experiments to compare the robustness of models quantitatively. To simulate the errors, we inject noise into the intermediate point $\rvx_r$ in each of the $35$ ODE steps: $\rvx_r=\rvx_r + \alpha\bm\epsilon_r$ where $\bm\epsilon_r\sim \gN(\bm 0, r/\sqrt{D}\mI)$, and $\alpha$ is a positive number controlling the amount of noise. \Figref{fig:alpha} demonstrates that as $\alpha$ increases, FID score exhibits a much slower degradation for smaller $D$. In particular, when $D{=}64,128$, the sample quality degrades gracefully. We further visualize the generated samples in Appendix~\ref{app:robust}. It shows that when $\alpha{=}0.2$, models with $D{=}64, 128$ can still produce clean images while the sampling process of diffusion models~($D{\to} \infty$) breaks down.

In addition to the controlled scenario, we conduct two more realistic experiments: \textbf{(1)} We introduce more estimation error of neural networks by applying post-training quantization~\cite{Sung2015ResiliencyOD}, which can directly compress neural networks without fine-tuning. Table~\ref{tab:quant} reports the FID score with varying quantization bit-widths for the convolution weight values. We can see that finite $D$s have better robustness than the infinite case, and a lower $D$ exhibits a larger performance gain when applying lower bit-widths quantization. \textbf{(2)} We increase the discretization error during sampling by using smaller NFEs, \ie larger sample steps. As shown in \Figref{fig:nfe}, gaps between $D{=}128$ and diffusion models gradually widen, indicating greater robustness against the discretization error. The rigidity issue of smaller $D$
also affects the robustness to discretization error, as $D{=}64$ is consistently inferior to $D{=}128$.

\begin{table}[h]
    \small
    \centering
    \caption{FID score versus quantization bit-widths on CIFAR-10.}
    \begin{tabular}{l c c c c c}
    \toprule
        \textit{Quantization bits:} & 9 & 8 & 7 & 6 & 5\\
        \midrule
        $D=64$ & 1.96& \bf 1.96& \bf 2.12& \bf 2.94& \bf 28.50\\
        $D=128$
& {1.93}
&  1.97 & 2.15
& {3.68} & {34.26} \\
        $D=2048$ & \bf{1.91}
& 1.97 & \bf 2.12
&5.67 & 47.02 \\ $D\to\infty$ 
&  1.97
& 2.04
& 2.16 & 5.91 & 50.09\\
        \bottomrule
    \end{tabular}
    \label{tab:quant}
\end{table}

\section{Conclusion and Future Directions}
We present a new family of physics-inspired generative models called PFGM++, by extending the dimensionality of augmented variable in PFGM from $1$ to $D \in \R^+$. Remarkably, PFGM++ includes diffusion models as special cases when $D{\to}\infty$. To address issues related to large batch training, we propose a perturbation-based objective. In addition, we show that $D$ effectively controls the robustness and rigidity in the PFGM++ family. Empirical results show that models with finite values of $D$ can perform better than previous state-of-the-art diffusion models, while also exhibiting improved robustness.

There are many potential avenues for future research in the PFGM++ framework. For example, it may be possible to identify the ``sweet spot" value of $D$ for different architectures and tasks by analyzing the behavior of errors. Since PFGM++ enables adjusting robustness, another direction is to apply aggressive network compression techniques, \ie pruning and low-bit training, to smaller $D$. Furthermore, there may be opportunities to develop stochastic samplers for PFGM++, with the reverse SDE in diffusion models as a special case. Lastly, as diffusion models have been highly optimized for image generation, the PFGM++ framework may show a greater advantage over its special case~(diffusion models) in emergent fields, such as biology data.

\clearpage

\section*{Acknowledgements}

YX and TJ acknowledge support from MIT-DSTA Singapore collaboration, from NSF Expeditions grant (award 1918839) “Understanding the World Through Code”, and from MIT-IBM Grand Challenge project. ZL and MT would like to thank the Center for Brains, Minds, and Machines (CBMM) for hospitality. ZL and MT are supported by The Casey and Family Foundation, the Foundational Questions Institute, the Rothberg Family Fund for Cognitive Science and IAIFI through NSF grant PHY-2019786. ST and TJ also acknowledge support from the ML for Pharmaceutical Discovery and Synthesis Consortium (MLPDS).
\bibliography{bib}
\bibliographystyle{icml2023}
\clearpage

\appendix

\newpage
\onecolumn
{\huge Appendix}

\def\E{{\bf E}}
\def\x{{\bf x}}
\def\r{{\bf r}}
\def\rhat{\hat{r}}

\section{Proofs}

\label{app:proofs}

\subsection{Proof of Theorem~\ref{thm:bijection}}
\label{app:proof-bij}
\thmbij*
\begin{proof}
Let $q_r(\rvx) \propto \int {r^D}/{\|\tilde{\mat{x}}-\tilde{\rvy}\|^{N+D}}p(\rvy)d\rvy$. We will show that $q_r\propto \int {r^D}/{\|\tilde{\mat{x}}-\tilde{\rvy}\|^{N+D}}p(\rvy)d\rvy$ is equal to the $r$-dependent marginal distribution $p_r$ by verifying (1) the starting distribution is correct when $r{=}0$; (2) the continuity equation holds, \ie $\partial_r q_r + \nabla_\rvx\cdot(q_r \mat{E}(\tilde{\mat{x}})_\rvx  /E(\tilde{\mat{x}})_{r})=0$. The starting distribution is $\lim_{r\to 0}q_r(\rvx) \propto \lim_{r\to 0}\int {r^D}/{\|\tilde{\mat{x}}-\tilde{\rvy}\|^{N+D}}p(\rvy)d\rvy\propto p(\rvx)$, which confirms that $q_r{=}p$. The continuity equation can be expressed as:
\begin{align*}
    &\partial_r q_r + \nabla_\rvx\cdot(q_r \mat{E}(\tilde{\mat{x}})_\rvx/E(\tilde{\mat{x}})_{r})\\
    &=  \partial_r\left(\int \frac{r^D}{\|\tilde{\mat{x}}-\tilde{\rvy}\|^{N+D}}p(\rvy)d\rvy\right)+ \nabla_\rvx\cdot\left(\int \frac{r^D}{\|\tilde{\mat{x}}-\tilde{\rvy}\|^{N+D}}{p}({\rvy}) d {\rvy}\frac{\int \frac{\tilde{\mat{x}}-\tilde{\rvy}}{\|\tilde{\mat{x}}-\tilde{\rvy}\|^{N+D}}{p}({\rvy}) d {\rvy}}{\int \frac{r}{\|\tilde{\mat{x}}-\tilde{\rvy}\|^{N+D}}{p}({\rvy}) d {\rvy}}\right)\\
    &=  \int \left(\frac{Dr^{D-1}}{{\|\tilde{\mat{x}}-\tilde{\rvy}\|^{N+D}}} - \frac{(N+D)r}{{\|\tilde{\mat{x}}-\tilde{\rvy}\|^{N+D-2}}}\right)p(\rvy)d\rvy+\nabla_\rvx\cdot\left(r^{D-1}{\int \frac{\tilde{\mat{x}}-\tilde{\rvy}}{\|\tilde{\mat{x}}-\tilde{\rvy}\|^{N+D}}{p}({\rvy}) d {\rvy}}\right)\\
    &=  \int \left(\frac{Dr^{D-1}}{{\|\tilde{\mat{x}}-\tilde{\rvy}\|^{N+D}}} - \frac{(N+D)r}{{\|\tilde{\mat{x}}-\tilde{\rvy}\|^{N+D-2}}}\right)p(\rvy)d\rvy+\nabla_\rvx\cdot\left(r^{D-1}{\int \frac{\tilde{\mat{x}}-\tilde{\rvy}}{\|\tilde{\mat{x}}-\tilde{\rvy}\|^{N+D}}{p}({\rvy}) d {\rvy}}\right)\\
    &=\int \left(\frac{Dr^{D-1}}{{\|\tilde{\mat{x}}-\tilde{\rvy}\|^{N+D}}} - \frac{(N+D)r}{{\|\tilde{\mat{x}}-\tilde{\rvy}\|^{N+D-2}}}\right)p(\rvy)d\rvy\\
    &\qquad +r^{D-1}\sum_{i=1}^N{\int \frac{\|\tilde{\mat{x}}-\tilde{\rvy}\|^{N+D} - \|\tilde{\mat{x}}-\tilde{\rvy}\|^{N+D-2}(\rvx_i -\rvy_i)^2(N+D)}{\|\tilde{\mat{x}}-\tilde{\rvy}\|^{2(N+D)}}{p}({\rvy}) d {\rvy}}\\
    &=\int \left(\frac{Dr^{D-1}}{{\|\tilde{\mat{x}}-\tilde{\rvy}\|^{N+D}}} - \frac{(N+D)r^{D+1}}{{\|\tilde{\mat{x}}-\tilde{\rvy}\|^{N+D-2}}}\right)p(\rvy)d\rvy\\
    &\qquad +r^{D-1}{\int \frac{N\|\tilde{\mat{x}}-\tilde{\rvy}\|^{N+D} - \|\tilde{\mat{x}}-\tilde{\rvy}\|^{N+D-2}\|\rvx-\rvy\|^2(N+D)}{\|\tilde{\mat{x}}-\tilde{\rvy}\|^{2(N+D)}}{p}({\rvy}) d {\rvy}}\\
    &=r^{D-1}\int \frac{\|\tilde{\mat{x}}{-}\tilde{\rvy}\|^{N{+}D}D-(N{+}D)r^{2}\|\tilde{\mat{x}}-\tilde{\rvy}\|^{N{+}D{-}2}+N\|\tilde{\mat{x}}{-}\tilde{\rvy}\|^{N{+}D}-\|\tilde{\mat{x}}{-}\tilde{\rvy}\|^{N{+}D{-}2}\|\rvx{-}\rvy\|^2(N{+}D)}{{\|\tilde{\mat{x}}{-}\tilde{\rvy}\|^{2(N{+}D)}}}p(\rvy)d\rvy\\
    &=r^{D-1}\int \frac{(N+D)(\|\tilde{\mat{x}}-\tilde{\rvy}\|^{N+D}-\|\tilde{\mat{x}}-\tilde{\rvy}\|^{N+D-2}\|\rvx-\rvy\|^2)-(N+D)r^{2}\|\tilde{\mat{x}}-\tilde{\rvy}\|^{N+D-2}}{{\|\tilde{\mat{x}}-\tilde{\rvy}\|^{2(N+D)}}}p(\rvy)d\rvy\\
    &=r^{D-1}\int \frac{(N+D)r^{2}\|\tilde{\mat{x}}-\tilde{\rvy}\|^{N+D-2}-(N+D)r^{2}\|\tilde{\mat{x}}-\tilde{\rvy}\|^{N+D-2}}{{\|\tilde{\mat{x}}-\tilde{\rvy}\|^{2(N+D)}}}p(\rvy)d\rvy\\
    &=0
\end{align*}
It means that $q_r$ satisfies the continuity equation for any $r\in \R_{\ge 0}$. Together, we conclude that $q_r=p_r$. Lastly, note that the terminal distribution is 
\begin{align*}
    \lim_{\rmax\to \infty}p_{\rmax}(\rvx) &\propto \lim_{\rmax\to \infty} \int \frac{\rmax^D}{\|\tilde{\mat{x}}-\tilde{\rvy}\|^{N+D}}p(\rvy)d\rvy=\lim_{\rmax\to \infty}\int \frac{\rmax^D}{(\|\rvx-\rvy\|^2+\rmax^2)^\frac{N+D}{2}}p(\rvy)d\rvy\\
    &= \lim_{\rmax\to \infty}\frac{\rmax^D}{(\|\rvx\|^2+\rmax^2)^\frac{N+D}{2}}  + \lim_{\rmax\to \infty}\int \left(\frac{\rmax^D}{(\|\rvx-\rvy\|^2+\rmax^2)^\frac{N+D}{2}}-\frac{\rmax^D}{(\|\rvx\|^2+\rmax^2)^\frac{N+D}{2}}\right)p(\rvy)d\rvy\\
    &= \lim_{\rmax\to \infty}\frac{\rmax^D}{(\|\rvx\|^2+\rmax^2)^\frac{N+D}{2}} \qquad \text{($p$ has a compact support)}
\end{align*}
\end{proof}

\subsection{Proof of Theorem~\ref{thm:minimizer}}
\label{app:thm-minimizer}
\thmmin*
\begin{proof}
The minimizer at $\tx$ in \Eqref{eq:new-D-aug} is
\begin{align*}
    f^*_{\theta}(\tilde{\rvx}) &= \int p_r({\rvy}|\rvx) ({{\tx}-\tilde{\rvy}} )d\tilde{\rvy}= \frac{\int p_r(\rvx|{{\rvy}})({\tilde{\rvx}-\tilde{\rvy}})p({\rvy}) d{\rvy}}{p_r(\rvx)}\numberthis \label{eq:minimizer-D}
\end{align*}
The choice of perturbation kernel is
\begin{align*}
    p_r(\rvx|\rvy) \propto \frac{1}{\|\tx - \tilde{\rvy}\|^{N+D}} =  \frac{1}{({\|\rvx-{\rvy}\|_2^2+ r^2})^\frac{N+D}{2}}
\end{align*}
By substituting the perturbation kernel in \Eqref{eq:minimizer-D}, we have:
\begin{align*}
    f^*_{\theta}(\tilde{\rvx})  &=\frac{\int \frac{{\tilde{\rvx}-\tilde{\rvy}}}{({\|\rvx-{\rvy}\|_2^2+ r^2})^\frac{N+D}{2}} p({\rvy}) d{\rvy}}{p_r(\rvx)}\\
    &= \frac{\int \frac{{\tilde{\rvx}-\tilde{\rvy}}}{{\|\tx-\tilde{\rvy}\|_2}^{N+D}} p({\rvy}) d{\rvy}}{p_r(\rvx)}\\
    &= ({S_{N+D-1}(1)}/{p_r(\rvx)})\mat{E}(\tx)
\end{align*}
\end{proof}

\subsection{Proof of Theorem~\ref{thm:inf-D-field}}
\label{app:proof-thmfield}
\thmfield*
\begin{proof}
The $\rvx$ component in the Poisson field can be re-expressed as
\begin{align*}
    \mat{E}({\tx})_{\rvx} &= \frac{1}{S_{N+D-1}(1)}\int \frac{{\mat{x}}-{\rvy}}{\|\tilde{\mat{x}}-\tilde{\rvy}\|^{N+D}}{p}({\rvy}) d {\rvy} \\
    &\propto \int p_r(\rvx|\rvy)(\rvx-\rvy)p(\rvy)d\rvy
\end{align*}
 where the perturbation kernel $p_r({\rvx}|\rvy)  \propto  {1}/{({\|\rvx-{\rvy}\|_2^2+ r^2})^\frac{N+D}{2}}$. The direction of the score can also be written down in a similar form: 
 \begin{align*}
     \nabla_\rvx \log p_{\sigma}(\rvx)=\frac{\int p_\sigma(\rvx|\rvy)\frac{\rvy-\rvx}{\sigma^2}p(\rvy)d\rvy}{p_\sigma(\rvx)}\propto \int p_\sigma(\rvx|\rvy)(\rvx-\rvy)p(\rvy)d\rvy
 \end{align*}
where $ p_{\sigma}(\rvx|\rvy) \propto \exp-\frac{\|\rvx-{\rvy}\|_2^2}{ 2\sigma^2}$. Since $p\in \gC^1$, and obviously $p_r(\rvx|\rvy) \in C^1$, then $\lim_{D\to \infty}\int p_r(\rvx|\rvy)(\rvx-\rvy)p(\rvy)d\rvy=\int \lim_{D\to \infty}p_r(\rvx|\rvy)(\rvx-\rvy)p(\rvy)d\rvy$. It suffices to prove that the perturbation kernel $p_r({\rvx}|\rvy)$ point-wisely converge to the Gaussian kernel $p_{\sigma}(\rvx|\rvy)$, \ie $\lim_{D\to \infty}p_r({\rvx}|\rvy) = p_\sigma({\rvx}|\rvy) $, to ensure $\mat{E}({\mat{x}})_{\rvx} \propto \nabla_\rvx \log p_{\sigma}(\rvx)$. Given $\forall \rvx,\rvy \in \mathbb{R}^{N}$, 
\begin{align*}
    \lim_{D\to \infty}p_r({\rvx}|\rvy) &\propto \lim_{D\to \infty}\frac{1}{({\|\rvx-\rvy\|_2^2+ r^2})^\frac{N+D}{2}}\\
    &=\lim_{D\to \infty}{({\|\rvx-\rvy\|_2^2+ r^2})^{-\frac{N+D}{2}}}\\
    &\propto\lim_{D\to \infty}{({1+\frac{\|\rvx-\rvy\|_2^2}{ r^2}})^{-\frac{N+D}{2}}}\\
    &=\lim_{D\to \infty}{({1+\frac{\|\rvx-\rvy\|_2^2}{ D\sigma^2}})^{-\frac{N+D}{2}}}\qquad \tag{$r=\sigma\sqrt{D}$}\\
    &= \lim_{D\to \infty}\exp\left(-\frac{N+D}{2} {\rm ln}(1+\frac{\|\rvx-\rvy\|_2^2}{ D\sigma^2}) \right)\\
    &= \lim_{D\to \infty}\exp\left(-\frac{N+D}{2} \frac{\|\rvx-\rvy\|_2^2}{ D\sigma^2} \right) \qquad \tag{ $\lim_{D\to \infty}\frac{\|\rvx-\rvy\|_2^2}{D\sigma^2}= 0 $}\\
   &= \exp-\frac{\|\rvx-\rvy\|_2^2}{ 2\sigma^2}\\
   &\propto  p_\sigma({\rvx}|\rvy) 
\end{align*}    
Hence $\lim_{D\to \infty}p_r({\rvx}|\rvy)=p_\sigma({\rvx}|\rvy)$, and we establish that $\mat{E}(\tx)_{\rvx} \propto \nabla_\rvx \log p_{\sigma}(\rvx)$. We can rewrite the drift term in the PFGM++ ODE as 
\begin{align*}
\lim_{\substack{D\to \infty\\r=\sigma\sqrt{D}}}\sqrt{D}\mat{E}(\tx)_{\rvx} /E(\tx)_{r}&=\lim_{\substack{D\to \infty\\r=\sigma\sqrt{D}}}\frac{\sqrt{D}\int p_r(\rvx|\rvy)(\rvx-\rvy)p(\rvy)d\rvy}{\int p_r(\rvx|\rvy)(-r)p(\rvy)d\rvy}\\
&=\lim_{\substack{D\to \infty\\r=\sigma\sqrt{D}}}\frac{\sqrt{D}\int p_r(\rvx|\rvy)(\rvy-\rvx)p(\rvy)d\rvy}{rp_r(\rvx)}\\
&= \lim_{\substack{D\to \infty\\r=\sigma\sqrt{D}}}\frac{\sqrt{D}\int p_\sigma(\rvx|\rvy)(\rvy-\rvx)p(\rvy)d\rvy}{rp_\sigma(\rvx)}\\
&= {\sigma \nabla_\rvx\log p_{\sigma}(\rvx)}\qquad\qquad \text{($\nabla_\rvx \log p_{\sigma}(\rvx)=\frac{\int p_\sigma(\rvx|\rvy)\frac{\rvy-\rvx}{\sigma^2}p(\rvy)d\rvy}{p_\sigma(\rvx)}$)}\numberthis \label{eq:EEinv}
\end{align*}
which establishes the first part of the theorem. For the second part, by the change-of-variable $d\sigma = dr/\sqrt{D}$, the PFGM++ ODE is 
\begin{align*}
    \lim_{\substack{D\to \infty\\r=\sigma\sqrt{D}}}\frac{\mathrm{d}\rvx}{\mathrm{d}\sigma} &= \frac{\mathrm{d}\rvx}{\mathrm{d}r} \cdot \frac{\mathrm{d}r}{\mathrm{d}\sigma}\\
    &=\lim_{\substack{D\to \infty\\r=\sigma\sqrt{D}}}\mat{E}(\tx)_{\rvx} \cdot E(\tx)_{r}^{-1} \cdot \sqrt{D}\\
    &= \lim_{\substack{D\to \infty\\r=\sigma\sqrt{D}}}\frac{\sigma \nabla_\rvx\log p_{\sigma}(\rvx)}{\sqrt{D}}\cdot \sqrt{D}\qquad \text{(by \Eqref{eq:EEinv})}\\
    &= {\sigma \nabla_\rvx\log p_{\sigma}(\rvx)}
\end{align*}
which is equivalent to the diffusion ODE.
\end{proof}

\subsection{Proof of Proposition~\ref{prop:obj}}
\label{app:proof-propobj}
\propobj*
\begin{proof}
    For $\forall \rvx \in \mathbb{R}^N$, the minimizer in PFGM++ objective~(\Eqref{eq:pfgmpp-obj}) at point $\tx=(\rvx, r)$ is
\begin{align*}
    f^*_{\theta,\textrm{PFGM++}}(\tilde{\rvx}) &= \lim_{\substack{D\to \infty\\r=\sigma\sqrt{D}}}\frac{\int p_r(\rvx|{{\rvy}})\frac{{{\rvx}-{\rvy}}}{r/\sqrt{D}}p({\rvy}) d{\rvy}}{p_r(\rvx)}\\
    &= \lim_{\substack{D\to \infty\\r=\sigma\sqrt{D}}}\frac{\int p_\sigma(\rvx|{{\rvy}})\frac{{{\rvx}-{\rvy}}}{r/\sqrt{D}}p({\rvy}) d{\rvy}}{p_\sigma(\rvx)} \qquad \tag{By Theorem~\ref{thm:inf-D-field}, $\lim_{D\to \infty}p_r(\rvx|{{\rvy}})=p_\sigma(\rvx|{{\rvy}})$}\\
    &= \frac{\int p_\sigma(\rvx|{{\rvy}})\frac{{{\rvx}-{\rvy}}}{\sigma}p({\rvy}) d{\rvy}}{p_\sigma(\rvx)}\numberthis \label{eq:minimizer-pfgmpp}
\end{align*}

On the other hand, the minimizer in denoising score matching at point $\rvx$ in noise level $\sigma=r/\sqrt{N+D}$ is
\begin{align*}
    f^*_{\theta,\textrm{DSM}}({\rvx},\sigma) = \frac{\int p_\sigma(\rvx|{{\rvy}})\frac{{{\rvx}-{\rvy}}}{\sigma}p({\rvy}) d{\rvy}}{p_\sigma(\rvx)}\numberthis \label{eq:minimizer-diffusion}
\end{align*}

Combining \Eqref{eq:minimizer-pfgmpp} and \Eqref{eq:minimizer-diffusion}, we have 
\begin{align*}
\lim_{\substack{D\to \infty\\r=\sigma\sqrt{D}}}f^*_{\theta,\textrm{PFGM++}}(\rvx, \sigma\sqrt{N+D}) =  f^*_{\theta,\textrm{DSM}}(\rvx, \sigma)
\end{align*}

\end{proof}

\section{Practical Sampling Procedures of Perturbation Kernel and Prior Distribution}
\label{app:sample-prior}
In this section, we discuss how to simple from the perturbation kernel $p_{r}(\rvx|{\rvy}) \propto {1}/{({\|\rvx-{\rvy}\|_2^2+ r^2})^\frac{N+D}{2}}$ in practice. We first decompose $p_r(\cdot|\rvy)$ in hyperspherical coordinates to $\gU_{\psi}(\psi)p_r(R)$, where $\gU_{\psi}$ is the uniform distribution over the angle component and the distribution of the perturbed radius $R=\|\rvx - \rvy\|_2$ is
\begin{align*}
    p_r(R) \propto \frac{R^{N-1}}{({R^2+ r^2})^\frac{N+D}{2}} \numberthis \label{eq:propto}
\end{align*}
The sampling procedure of the radius distribution encompasses three steps:
\begin{align*}
    &R_1 \sim \textrm{Beta}(\alpha=\frac{N}{2}, \beta=\frac{D}{2})\\
    &R_2=\frac{R_1}{1-R_1}\\
    &R_3= \sqrt{r^2 R_2}
\end{align*}
Next, we prove that $p(R_3) = p_r(R_3)$. Note that the pdf of the inverse beta distribution is
\begin{align*}
    p(R_2) \propto R_2^{\frac{N}{2}-1}(1+R_2)^{-\frac{N+D}{2}}
\end{align*}
By change-of-variable, the pdf of $R_3=\sqrt{r_{max}^2 R_2}$ is
\begin{align*}
    p(R_3) &\propto R_2^{{\frac{N}{2}-1}}(1+R_2)^{-\frac{N}{2}-\frac{D}{2}}*\frac{2R_3}{r_{max}^2}\\
    &\propto\frac{R_3R_2^{\frac{N}{2}-1}}{(1+R_2)^{\frac{N+D}{2}}}\\
    &= \frac{(R_3/r)^{N-1}}{(1+(R_3^2/r^2))^{\frac{N+D}{2}}}\\
    &\propto  \frac{R_3^{N-1}}{(1+(R_3^2/r^2))^{\frac{N+D}{2}}}\\
    &\propto  \frac{R_3^{N-1}}{(r^2+R_3^2)^{\frac{N+D}{2}}} \propto p_{r}(R_3)\qquad \textrm{(By \Eqref{eq:propto})}
\end{align*}

Note that $R_1$ has mean $\frac{N}{N+D}$ and variance $O(\frac{ND}{(N+D)^3})$. Hence when $D=O(N)$,  $p_{r}(R)$ would highly concentrate on a specific value, resolving the heavy-tailed problem. We can sample the uniform angel component by $\rvu = \rvw/\|\rvw\|, \rvw \sim \gN(\mathbf{0}, \mI_{N \times N})$. Together, sampling from the perturbation kernel $p_{r}(\rvx|{\rvy}) $ is equivalent to setting $\rvx = \rvy + R_3\rvu$. On the other hand, the prior distribution is 
\begin{align*}
    p_{\rmax}(\rvx) \propto \lim_{\rmax\to \infty} \int {\rmax^D}/{\|\tilde{\mat{x}}-\tilde{\rvy}\|^{N+D}}p(\rvy)d\rvy= \lim_{\rmax\to \infty}{\rmax^D}/{(\|\rvx\|^2+\rmax^2)^\frac{N+D}{2}}
\end{align*}
We observe that $p_{\rmax}(\rvx)$ the same as the perturbation kernel $p_{\rmax}(\rvx|{\rvy}=\bf{0})$. Hence we can sample from the prior following $\rvx = R_3\rvu$ with $R_3, \rvu$ defined above and $r=\rmax$.

\section{$r=\sigma\sqrt{D}$ for Phase Alignment}
\subsection{Analysis}
\label{app:phase-align}

In this section, we examine the phase of intermediate marginal distribution $p_r$ under different $D$s to derive an alignment method for hyper-parameters. Consider a $N$-dimensional dataset $\gD$ in which the average distance to the nearest neighbor is about $l$. We consider an arbitrary datapoint $\rvx_1 \in \gD$ and denote its nearest neighbor as $\rvx_2$. We assume $\|\rvx_1 - \rvx_2\|_2=l$, and uniform prior on $\gD$.

To characterize the phases of $p_r, \forall r>0$, we study the perturbation point $\rvy \sim p_r(\rvy|\rvx_1)$. According to Appendix~\ref{app:sample-prior}, the distance $\|\rvx_1-\rvy\|$ is roughly $r\sqrt{\frac{N}{D-1}}$. Since $p_r(\rvy|\rvx_1)$ is isotropic, with high probability, the two vectors $\rvy-\rvx_1, \rvx_2-\rvx_1$ are approximately orthogonal. In particular, the vector product $(\rvy - \rvx_1)^T(\rvx_1 - \rvx_2)= O(\frac{1}{\sqrt{N}}\|\rvy - \rvx_1\|\|\rvx_1 - \rvx_2\|)=O(\frac{rl}{\sqrt{D}})$ w.h.p. It reveals that $\|\rvy-\rvx_2\| = \sqrt{l^2+r^2\frac{N}{D-1}+O(\frac{rl}{\sqrt{D}})}$. \Figref{fig:align-analysis} depicts the relative positions of $\rvx_1, \rvx_2$ and the perturbed point $\rvy$.

\begin{figure*}[t]
\centering    \includegraphics[width=0.8\textwidth]{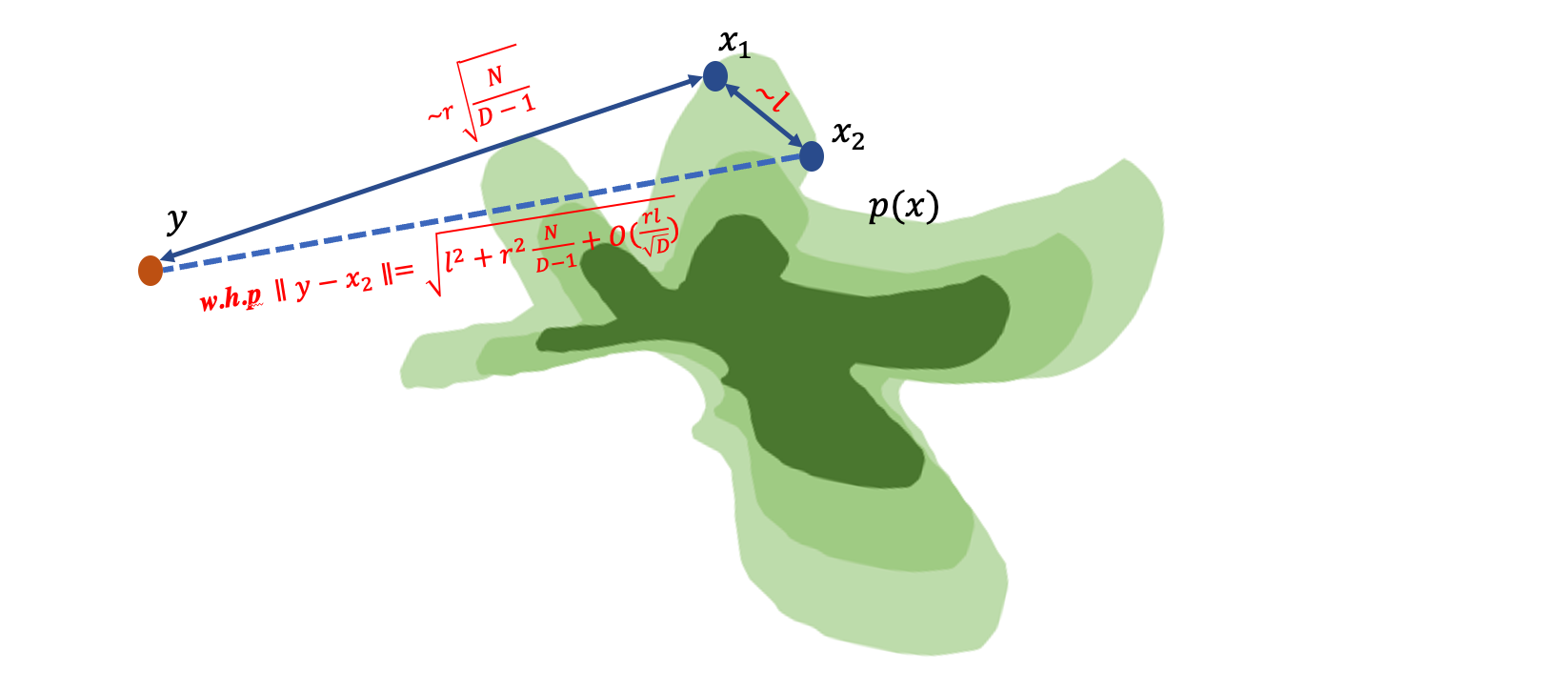}
    \caption{Illustration of the phase alignment analysis}
    \label{fig:align-analysis}
\end{figure*}

The ratio of the posterior of the $\rvx_2$ and $\rvx_1$  --- $\frac{p_r(\rvx_2|\rvy)}{p_r(\rvx_1|\rvy)}$ ---  is an indicator of different phases of field~\cite{Xu2023StableTF}: point in the nearer field tends to have a smaller ratio. Indeed, the ratio would gradually decay from $1$ to $0$ when moving from the far to the near field.  We can calculate the ratio of the coefficients after approximating the distance $\|\rvy-\rvx_2\|$:
\begin{align*}
    \frac{p_r(\rvx_2|\rvy)}{p_r(\rvx_1|\rvy)}=\frac{p_r(\rvy|\rvx_2)}{p_r(\rvy|\rvx_1)} &=  \left(\frac{l^2+r^2\frac{N}{D-1}+O(\frac{rl}{\sqrt{D}})+r^2}{r^2\frac{N}{D-1}+r^2}\right)^{\frac{N+D}{2}}\\
    &=\left(1 + \frac{l^2+O(\frac{rl}{\sqrt{D}})}{r^2\frac{N}{D-1}+r^2}\right)^{\frac{N+D}{2}}\\
    & = \exp{\left({\rm ln}(1+\frac{l^2+O(\frac{rl}{\sqrt{D}})}{r^2\frac{N}{D-1}+r^2})\cdot \frac{N+D}{2}\right)}\\
    & \approx \exp{\left(\frac{l^2+O(\frac{rl}{\sqrt{D}})}{r^2\frac{N}{D-1}+r^2} \cdot\frac{N+D}{2}\right)}\\
    &= \exp{\left(\frac{l^2+O(\frac{rl}{\sqrt{D}})}{r^2} \cdot\frac{N+D}{2(N+D-1)}\cdot(D-1)\right)}\\
    &\approx \exp{\left(\frac{l^2+O(\frac{rl}{\sqrt{D}})}{r^2}\cdot D\right)}\numberthis \label{eq:ratio}
\end{align*}

Hence the relation $r\propto \sqrt{D}$ should hold to keep the ratio invariant of the parameter $D$. On the other hand, by Theorem~\ref{thm:inf-D-field} we know that $p_{\sigma}$ is equivalent to $p_{r=\sigma\sqrt{D}}$ when $D\to \infty$. To achieve phase alignment on the dataset, one should roughly set $r=\sigma\sqrt{D}$.

\subsection{Practical Hyperparameter Transfer from Diffusion Models}
\label{app:transfer-diff}

\subsubsection{Transfer EDM training and sampling}

We list out and compare the EDM training algorithm~(Alg~\ref{algorithm-edm}) and the PFGM++ with transferred hyper-parameters~(Alg~\ref{algorithm-edm-pfgmpp}). The major modification is to replace the Gaussian noise $\rvn_i\sim \gN(0, \sigma^2 \mI)$ with the additive noise $R_i \rvv_i\sim \gU_{\psi}(\psi)p_r(R)$, where $r=\sigma\sqrt{D}$. We highlight the major modifications in \textcolor{blue}{blue}.

We also show the sampling algorithms of EDM~(Alg~\ref{algorithm-edm-sample}) and PFGM++~(Alg~\ref{algorithm-edm-pfgmpp-sample}). Note that we only change the prior sampling process while the for-loop is identical for both algorithms, since EDM~\citep{Karras2022ElucidatingTD} sets $\sigma=t$, and $\frac{\mathrm{d}\rvx}{\mathrm{d}r}=\frac{\rvx-f_{\theta}(\rvx, r)}{r}=\frac{\rvx-f_{\theta}(\rvx, r)}{\sigma\sqrt{D}}=\frac{\mathrm{d}\rvx}{\sqrt{D}\mathrm{d}\sigma}=\frac{\mathrm{d}\rvx}{\mathrm{d}\sigma}\frac{\mathrm{d}\sigma}{\mathrm{d}r}=\frac{\mathrm{d}\rvx}{\mathrm{d}\sigma}=\frac{\mathrm{d}\rvx}{\mathrm{d}t}$. Thus we can use the original samplers of EDM without further modification.

\begin{minipage}{0.46\textwidth}
\vspace{-46pt}
\begin{algorithm}[H]
    \centering
    \caption{EDM training}\label{algorithm-edm}
    \begin{algorithmic}[1]
        \STATE Sample a batch of data $\{\rvy_i\}_{i=1}^\gB$ from $p(\rvy)$
        \STATE Sample standard deviations $\{\sigma_i\}_{i=1}^\gB$ from $p(\sigma)$ 
        \STATE Sample noise vectors $\{\rvn_i \sim \gN(0, \sigma_i^2 \mI)\}_{i=1}^\gB$
        \STATE Get perturbed data $\{\hat{\rvy}_i = \rvy_i + \rvn_i\}_{i=1}^\gB$
        \STATE Calculate loss $\ell(\theta) =  \sum_{i=1}^\gB \lambda (\sigma_i)\|f_\theta(\hat{\rvy}_i, \sigma_i)-\rvy_i\|_2^2$
        \STATE Update the network parameter $\theta$ via Adam optimizer
    \end{algorithmic}
\end{algorithm}
\end{minipage}
\hfill
\begin{minipage}{0.50\textwidth}
\begin{algorithm}[H]
    \centering
    \caption{PFGM++ training with hyperparameter transferred from EDM}\label{algorithm-edm-pfgmpp}
    \begin{algorithmic}[1]
        \STATE Sample a batch of data $\{\rvy_i\}_{i=1}^\gB$ from $p(\rvy)$
        \STATE Sample standard deviations $\{\sigma_i\}_{i=1}^\gB$ from $p(\sigma)$ 
        \STATE \textcolor{blue}{Sample $r$ from $p_r$: $\{r_i = \sigma_i\sqrt{D}\}_{i=1}^\gB$}
        \STATE  \textcolor{blue}{Sample radiuses $\{R_i \sim p_{r_i}(R)\}_{i=1}^\gB$}
        \STATE  \textcolor{blue}{Sample uniform angles $\{\rvv_i =\frac{\rvu_i}{\|\rvu_i\|_2}\}_{i=1}^\gB$, with $\rvu_i \sim \gN(\bf{0}, \mI)$}
        \STATE \textcolor{blue}{Get perturbed data $\{\hat{\rvy}_i = \rvy_i + R_i\rvv_i\}_{i=1}^\gB$}
        \STATE {Calculate loss $\ell(\theta) =  \sum_{i=1}^\gB \lambda (\sigma_i)\|f_\theta(\hat{\rvy}_i, \sigma_i)-\rvy_i\|_2^2$}
        \STATE Update the network parameter $\theta$ via Adam optimizer
    \end{algorithmic}
\end{algorithm}
\end{minipage}

\begin{minipage}{0.46\textwidth}
\vspace{-47pt}
\begin{algorithm}[H]
    \centering
    \caption{EDM sampling~(Heun’s $2^{\textrm{nd}}$ order method)}\label{algorithm-edm-sample}
    \begin{algorithmic}[1]
\STATE $\rvx_0 \sim \gN(\bm{0}, \sigma_{\textrm{max}}^2\mI)$
\FOR{$i=0, \dots, T-1$}
\STATE $\rvd_i = ({\rvx_i -f_{\theta}(\rvx_i,t_i)})/{t_i}$
\STATE $\rvx_{i+1} =\rvx_i + (t_{i+1}-t_i)\rvd_i$
\IF{$t_{i+1}>0$}
\STATE $\rvd_i' = ({\rvx_{i+1} -f_{\theta}(\rvx_{i+1},t_{i+1})})/{t_{i+1}}$
\STATE $\rvx_{i+1}=\rvx_i + (t_{i+1}-t_i)(\frac12\rvd_i+\frac12\rvd_i')$
\ENDIF
\ENDFOR
    \end{algorithmic}
\end{algorithm}
\end{minipage}
\hfill
\begin{minipage}{0.50\textwidth}
\begin{algorithm}[H]
    \centering
    \caption{PFGM++ training with hyperparameter transferred from EDM}\label{algorithm-edm-pfgmpp-sample}
    \begin{algorithmic}[1]
 \STATE \textcolor{blue}{Set  $\rmax = \sigma_{\textrm{max}}\sqrt{D}$}
        \STATE  \textcolor{blue}{Sample radius $R \sim p_{\rmax}(R)$
       and uniform angle $\rvv =\frac{\rvu}{\|\rvu\|_2}$, with $\rvu \sim \gN(\bf{0}, \mI)$}
        \STATE \textcolor{blue}{Get initial data $\rvx_0 = R\rvv$}
\FOR{$i=0, \dots, T-1$}
\STATE $\rvd_i = ({\rvx_i -f_{\theta}(\rvx_i,t_i)})/{t_i}$
\STATE $\rvx_{i+1} =\rvx_i + (t_{i+1}-t_i)\rvd_i$
\IF{$t_{i+1}>0$}
\STATE $\rvd_i' = ({\rvx_{i+1} -f_{\theta}(\rvx_{i+1},t_{i+1})})/{t_{i+1}}$
\STATE $\rvx_{i+1}=\rvx_i + (t_{i+1}-t_i)(\frac12\rvd_i+\frac12\rvd_i')$
\ENDIF
\ENDFOR
    \end{algorithmic}
\end{algorithm}
\end{minipage}

\subsubsection{Transfer DDPM~(continuous) training and sampling}

Here we demonstrate the ``zero-shot" transfer of hyperparameters from DDPM to PFGM++, using the $r=\sigma\sqrt{D}$ formula. We highlight the modifications in \textcolor{blue}{blue}. In particular, we list the DDPM training/sampling algorithms~(Alg~\ref{algorithm-ddpm}/Alg~\ref{algorithm-ddpm-sample}), and their counterparts in PFGM++~(Alg~\ref{algorithm-ddpm-pfgmpp}/Alg~\ref{algorithm-ddpm-pfgmpp-sample}) for comparions. Let $\beta_T$ and $\beta_1$ be the maximum/minimum values of $\beta$ in DDPM~\cite{Ho2020DenoisingDP}. Similar to \citet{Song2021ScoreBasedGM}, we denote $\alpha_{t}=e^{-\frac12t^2(\bar{\beta}_{\textrm{max}}-\bar{\beta}_{\textrm{min}}) - t\bar{\beta}_{\textrm{min}}}$, with $\bar{\beta}_{\textrm{max}}=\beta_T  \cdot T$ and $\bar{\beta}_{\textrm{min}}=\beta_1\cdot T$. For example, on CIFAR-10, $\bar{\beta}_{\textrm{min}}=1e-1$ and $\bar{\beta}_{\textrm{max}}=20$ with $T=1000$. We would like to note that the $t_i$s in the sampling algorithms~(Alg~\ref{algorithm-ddpm-sample} and Alg~\ref{algorithm-ddpm-pfgmpp-sample}) monotonically decrease from $1$ to $0$ as $i$ increases.

\begin{minipage}{0.46\textwidth}
\vspace{-69pt}
\begin{algorithm}[H]
    \centering
    \caption{DDPM training}\label{algorithm-ddpm}
    \begin{algorithmic}[1]
        \STATE Sample a batch of data $\{\rvy_i\}_{i=1}^\gB$ from $p(\rvy)$
        \STATE Sample time $\{t_i{=}t_i'/T\}_{i=1}^\gB$ with $t_i'{\sim} \gU(\{1, \dots, T\})$ 
        \STATE Get perturbed data $\{\hat{\rvy}_i = \sqrt{\alpha_{t_i}}\rvy_i +\sqrt{1-\alpha_{t_i}}\bm{\epsilon}_i\}_{i=1}^\gB$,
        where $\bm{\epsilon}_i \sim \gN(\bf{0}, \mI)$
        \STATE Calculate loss $\ell(\theta) =  \sum_{i= 1}^\gB \lambda (t_i)\|f_\theta(\hat{\rvy}_i, t_i)-\bm{\epsilon}_i\|_2^2$
        \STATE Update the network parameter $\theta$ via Adam optimizer
    \end{algorithmic}
\end{algorithm}
\end{minipage}
\hfill
\begin{minipage}{0.50\textwidth}
\begin{algorithm}[H]
    \centering
    \caption{PFGM++ training with hyperparameter transferred from DDPM}\label{algorithm-ddpm-pfgmpp}
    \begin{algorithmic}[1]
        \STATE Sample a batch of data $\{\rvy_i\}_{i=1}^\gB$ from $p(\rvy)$
       \STATE Sample time $\{t_i\}_{i=1}^\gB$ from $\gU[0,1]$
       \STATE \textcolor{blue}{Get $\sigma_{i}$ from $t_i$: $\{\sigma_{i}=\sqrt{\frac{1-\alpha_{t_i}}{\alpha_{t_i}}}\}$}
        \STATE \textcolor{blue}{Sample $r$ from $p_r$: $\{r_i = \sigma_i\sqrt{D}\}_{i=1}^\gB$}
        \STATE  \textcolor{blue}{Sample radiuses $\{R_i \sim p_{r_i}(R)\}_{i=1}^\gB$}
        \STATE  \textcolor{blue}{Sample uniform angles $\{\rvv_i =\frac{\rvu_i}{\|\rvu_i\|_2}\}_{i=1}^\gB$, with $\rvu_i \sim \gN(\bf{0}, \mI)$}
        \STATE \textcolor{blue}{Get perturbed data $\{\hat{\rvy}_i = \sqrt{\alpha_{t_i}}(\rvy_i + R_i\rvv_i)\}_{i=1}^\gB$}
        \STATE \textcolor{blue}{Calculate loss $\ell(\theta) =  \sum_{i=1}^\gB \lambda (t_i)\|f_\theta(\hat{\rvy}_i, t_i)-\frac{\sqrt{D}R_i\rvv_i}{r}\|_2^2$}
        \STATE Update the network parameter $\theta$ via Adam optimizer
    \end{algorithmic}
\end{algorithm}
\end{minipage}

\begin{minipage}{0.46\textwidth}
\vspace{-44pt}
\begin{algorithm}[H]
    \centering
    \caption{DDIM sampling}\label{algorithm-ddpm-sample}
    \begin{algorithmic}[1]
\STATE $\rvx_T \sim \gN(\bm{0}, \mI)$
\FOR{$i=T, \dots, 1$}
\STATE $\rvx_{i-1} =\sqrt{\frac{\alpha_{t_{i-1}}}{\alpha_{t_i}}}\rvx_i $\\
$ \qquad +(\sqrt{1-\alpha_{t_{i-1}}}{-}\sqrt{\frac{\alpha_{t_{i-1}}}{\alpha_{t_i}}}\sqrt{1-\alpha_{t_i}})f_{\theta}(\rvx_i, t_i)$
\ENDFOR
    \end{algorithmic}
\end{algorithm}
\end{minipage}
\hfill
\begin{minipage}{0.50\textwidth}
\begin{algorithm}[H]
    \centering
    \caption{PFGM++ sampling transferred from DDIM}\label{algorithm-ddpm-pfgmpp-sample}
    \begin{algorithmic}[1]
     \STATE \textcolor{blue}{Set $\sigma_{\textrm{max}}=\sqrt{\frac{1-\alpha_{1}}{\alpha_{1}}}, \rmax = \sigma_{\textrm{max}}\sqrt{D}$}
        \STATE  \textcolor{blue}{Sample radius $R \sim p_{\rmax}(R)$
       and  uniform angle $\rvv =\frac{\rvu}{\|\rvu\|_2}$, with $\rvu \sim \gN(\bf{0}, \mI)$}
        \STATE \textcolor{blue}{Get initial data $\rvx_T = \sqrt{\alpha_1}R\rvv$}
\FOR{$i=T, \dots, 1$}
\STATE $\rvx_{i-1} =\sqrt{\frac{\alpha_{t_{i-1}}}{\alpha_{t_i}}}\rvx_i $\\
$ \qquad +(\sqrt{1-\alpha_{t_{i-1}}}{-}\sqrt{\frac{\alpha_{t_{i-1}}}{\alpha_{t_i}}}\sqrt{1-\alpha_{t_i}})f_{\theta}(\rvx_i, t_i)$
\ENDFOR
    \end{algorithmic}
\end{algorithm}
\end{minipage}



\section{Experimental Details}
\label{app:exp}
We show the experimental setups in section~\ref{sec:benefits}, as well as the training, sampling, and evaluation details for PFGM++. All the experiments are run on four NVIDIA A100 GPUs or eight NVIDIA V100 GPUs.

\subsection{Experiments for the Analysis in Sec~\ref{sec:benefits}}
\label{app:be-exp}
In the experiments of section~\ref{sec:diffusion} and section~\ref{sec:behavior}, we need to access the posterior $p_{0|r}(\rvy|\rvx)\propto p_r(\rvx|\rvy)p(\rvy)$ to calculate the mean TVD. We sample a large batch $\{\rvy_i\}_{i=1}^n$ with $n=1024$ on CIFAR-10 to empirically approximate the posterior:
\begin{align*}
    p_{0|r}(\rvy_i|\rvx) = \frac{p_r(\rvx|\rvy_i)p(\rvy_i)}{p_r(\rvx)}\approx \frac{p_r(\rvx|\rvy_i)}{\sum_{j=1}^n p_r(\rvx|\rvy_j)}=\frac{{1}/{({\|\rvx-{\rvy_i}\|_2^2+ r^2})^\frac{N+D}{2}}}{\sum_{j=1}^n {1}/{({\|\rvx-{\rvy_j}\|_2^2+ r^2})^\frac{N+D}{2}}}
\end{align*}

We sample a large batch of $256$ to approximate all the expectations in section~\ref{sec:benefits}, such as the average TVDs.
\subsection{Training Details}

We borrow the architectures, preconditioning techniques, optimizers, exponential moving average~(EMA) schedule, and hyper-parameters from previous state-of-the-art diffusion model EDM~\cite{Karras2022ElucidatingTD}. We apply the alignment method in section~\ref{sec:diffusion} to transfer their well-tuned hyper-parameters. 

For architecture, we use the improved NCSN++~\cite{Karras2022ElucidatingTD} for the CIFAR-10 dataset~(batch size $512$), and the improved DDPM++ for the FFHQ dataset~(batch size $256$). For optimizers, following EDM, we adopt the Adam optimizer with a learning rate of $10e-4$. We further incorporate the EMA schedule, learning rate warm-up, and data augmentations in EDM. Please refer to Appendix F in EDM paper~\cite{Karras2022ElucidatingTD} for details.

The most prominent improvements in EDM are the preconditioning and the new training distribution for $\sigma$, \ie $p(\sigma)$. Specifically, adding these two techniques to the vanilla diffusion objective~(\Eqref{eq:diffusion-obj}), their effective training objective can be written as:
\begin{align*}
    \E_{\sigma \sim p(\sigma)}\lambda(\sigma)c_{\textrm{out}}(\sigma)^2\E_{p({\rvy})}\E_{p_{\sigma}({\rvx}|{\rvy})}\left[\Big\lVert F_{\theta}({c_{\textrm{in}}(\sigma) \cdot \rvx}, c_{\textrm{noise}}(\sigma))- \frac{1}{c_{\textrm{out}}(\sigma)}(\rvy-c_{\textrm{skip}}(\sigma)\cdot \rvx)\Big\rVert_2^2\right] \numberthis \label{eq:edm-obj}
\end{align*}
with the predicted normalized score function in the vanilla diffusion objective~(\Eqref{eq:diffusion-obj}) re-parameterized as 
\begin{align*}
    f_{\theta}(\rvx, \sigma) = \frac{c_{\textrm{skip}}(\sigma)\rvx+c_{\textrm{out}}(\sigma)F_{\theta}({c_{\textrm{in}}(\sigma) \rvx}, c_{\textrm{noise}}(\sigma)) - x}{\sigma} \approx \sigma\nabla_\rvx \log p_\sigma(x)
\end{align*}
$c_{\textrm{in}}(\sigma)=1/\sqrt{\sigma^2+\sigma_{\textrm{data}}^2}, c_{\textrm{out}}(\sigma)=\sigma\cdot\sigma_{\textrm{data}}/\sqrt{\sigma^2+\sigma_{\textrm{data}}^2}, c_{\textrm{skip}}(\sigma)=\sigma_{\textrm{data}}^2/(\sigma^2+\sigma_{\textrm{data}}^2), c_{\textrm{noise}}(\sigma)=\frac14 \ln(\sigma)$, with $\sigma_{\textrm{data}}=0.5$. $\{c_{\textrm{in}}(\sigma),c_{\textrm{out}}(\sigma),c_{\textrm{skip}}(\sigma),c_{\textrm{data}},c_{\textrm{noise}}(\sigma)\}$ are all the hyper-parameters in the preconditioning. The training distribution $p(\sigma)$ is the log-normal distribution with $\ln(\sigma) \sim \gN(-1.2, {1.2}^2)$, and the loss weighting $\lambda(\sigma) = 1/c_{\textrm{out}}(\sigma)^2$. 

Recall that the hyper-parameter alignment rule $r=\sigma\sqrt{D}$ can transfer the hyper-parameter from diffusion models~($D{\to} \infty$) to finite $D$s. Hence we can directly set $\sigma=r/\sqrt{D}$ in those hyper-parameters for preconditioning. In addition, the training distribution $p(r)$ can be derived via the change-of-variable formula, \ie $p(r)=p(\sigma=r/\sqrt{D})/\sqrt{D}$. The final PFGM++ objective after incorporating these techniques into \Eqref{eq:pfgmpp-obj} is:
\begin{align*}
    \E_{r \sim p(r)}\lambda(r/\sqrt{D})c_{\textrm{out}}(r/\sqrt{D})^2\E_{p({\rvy})}\E_{p_{r}({\rvx}|{\rvy})}\left[\Big\lVert F_{\theta}({c_{\textrm{in}}(r/\sqrt{D}) \cdot \rvx}, c_{\textrm{noise}}(r/\sqrt{D}))- \frac{1}{c_{\textrm{out}}(\sigma)}(\rvy-c_{\textrm{skip}}(r/\sqrt{D})\cdot \rvx)\Big\rVert_2^2\right]
\end{align*}
with the predicted normalized electric field in the vanilla PFGM++ objective~(\Eqref{eq:pfgmpp-obj}) re-parameterized as 
\begin{align*}
    f_{\theta}(\tx) = \frac{c_{\textrm{skip}}(r/\sqrt{D})\rvx+c_{\textrm{out}}(r/\sqrt{D})F_{\theta}({c_{\textrm{in}}(r/\sqrt{D}) \rvx}, c_{\textrm{noise}}(r/\sqrt{D})) - x}{r/\sqrt{D}} \approx {\sqrt{D}}\frac{\mat{E}(\tilde{\mat{x}})_{\rvx}}{E(\tilde{\mat{x}})_{r}}
\end{align*}
\subsection{Sampling Details}

For sampling, following EDM~\cite{Karras2022ElucidatingTD}, we also use Heun's $2^{\textrm{nd}}$ method~(improved Euler method)~\cite{Ascher1998ComputerMF} as the ODE solver for $\mathrm{d}\rvx /\mathrm{d}r = \mat{E}(\tilde{\mat{x}})_\rvx /E(\tilde{\mat{x}})_{r}= f_\theta(\tx) / \sqrt{D}$.

We adopt the same parameterized scheme in EDM to determine the evaluation points during $N$-step ODE sampling:
\begin{align*}
    r_i = ({\rmax}^\frac{1}{\rho} + \frac{i}{N-1}({\rmin}^\frac{1}{\rho}-{\rmax}^\frac{1}{\rho}))^\rho \quad \textrm{and} \quad r_N=0
\end{align*}
where $\rho$ controls the relative density of evaluation points in the near field. We set $\rho=7$ as in EDM, and $\rmax=\sigma_{\textrm{max}}\sqrt{D}=80\sqrt{D}, \rmin=\sigma_{\textrm{min}}\sqrt{D}=0.002\sqrt{D}$~($\sigma_{\textrm{max}},\sigma_{\textrm{min}}$ are the hyper-parameters in EDM, controlling the starting/terminal evaluation points) following the $r=\sigma\sqrt{D}$ alignment rule.
\subsection{Evaluation Details}
\label{app:eval}
For the evaluation, we compute the Fréchet distance between 50000 generated samples and the pre-computed statistics of CIFAR-10 and FFHQ. On CIFAR-10, we follow the evaluation protocol in EDM~\cite{Karras2022ElucidatingTD}, which repeats the generation three times with different seeds for each checkpoint and reports the minimum FID score. However, we observe that the FID score has a large fluctuation across checkpoints, and the minimum FID score of EDM in our re-run experiment does not align with the original results reported in \cite{Karras2022ElucidatingTD}. \Figref{fig:ffhq} shows that the FID score could have a variation of $\pm 0.2$ during the training of a total of 200 million images~\cite{Karras2022ElucidatingTD}. To better evaluate the model performance, Table~\ref{tab:ffhq} reports the average FID over the Top-3 checkpoints instead. In \Figref{fig:ffhq-avg}, we further demonstrate the moving average of the FID score with a window of $10000$K images. It shows that $D=2048$ consistently outperforms other baselines in the same training iterations, in agreement with the results in Table~\ref{tab:ffhq}.

\begin{figure*}
\centering
    \subfigure[w/o moving average]{\includegraphics[width=0.8\textwidth]{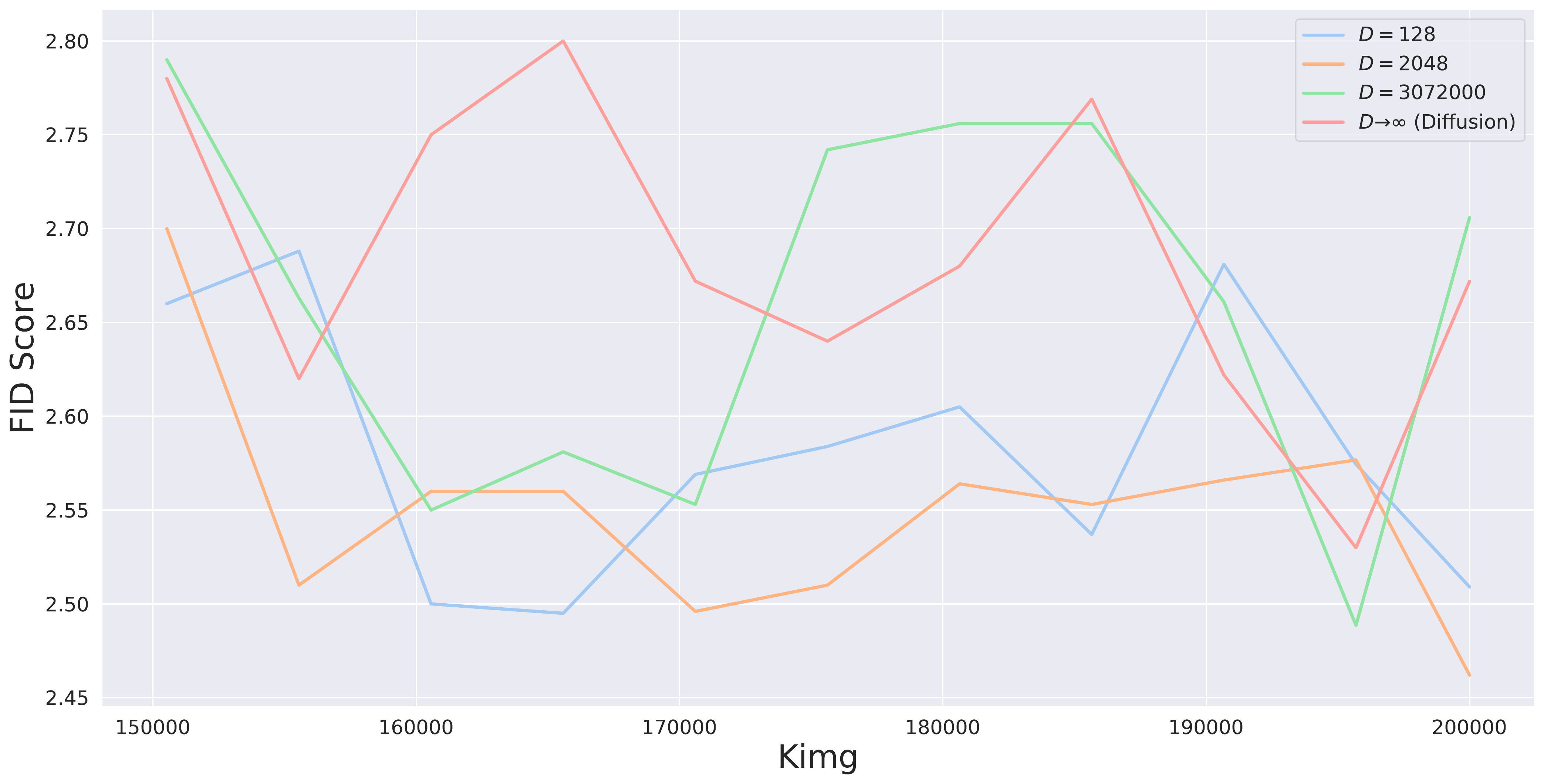}\label{fig:ffhq}}
        \subfigure[w/ moving average]{\includegraphics[width=0.8\textwidth]{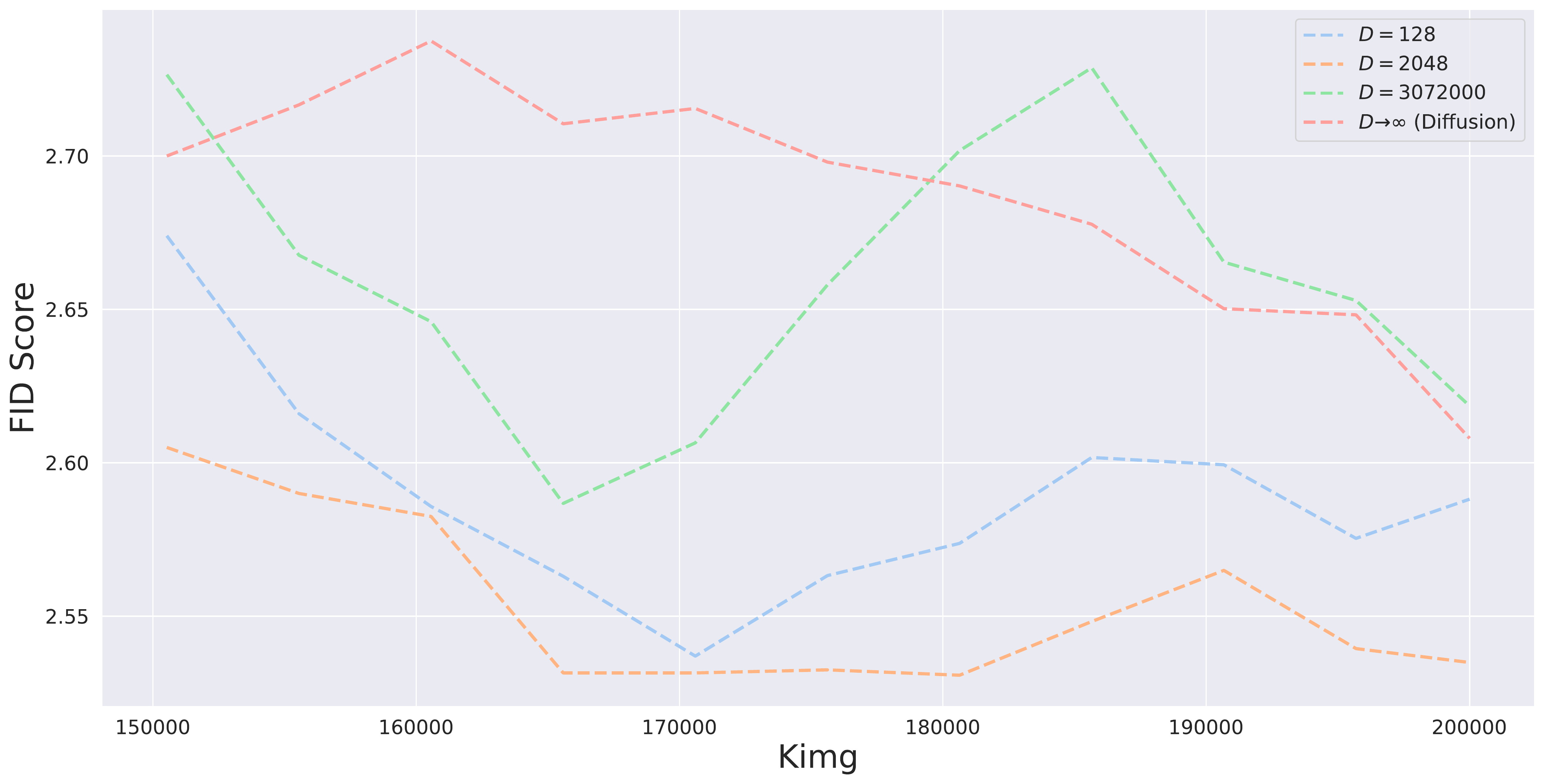}\label{fig:ffhq-avg}}
        \caption{FID score in the training course when varying $D$, \textbf{(a)} w/o and \textbf{(b)} w/ moving average.}
\end{figure*}

\subsection{Experiments for Robustness}
\label{app:robust-exp}

\paragraph{Controlled experiments with $\alpha$} In the controlled noise setting, we inject noise into the intermediate point $\rvx_r$ in each of the $35$ ODE steps by $\rvx_r=\rvx_r + \alpha\bm\epsilon_r$ where $\bm\epsilon_r\sim \gN(\bm 0, r/\sqrt{D}\mI)$. Since $p_r$ has roughly the same phase as $p_{\sigma=r/\sqrt{D}}$ in diffusion models, we pick $r/\sqrt{D}$ standard deviation of $\bm \epsilon_r$ when the intermediate step is $r$.

\paragraph{Post-training quantization} In the post-training quantization experiments on CIFAR-10, we quantize the weights of convolutional layers excluding the $32\times 32$ layers, as we empirically observe that these input/output layers are more critical for sample quality.
\section{Extra Experiments}

\subsection{Stable Target Field}
\label{app:stf}

\citet{Xu2023StableTF} propose a Stable Target Field objective for training the diffusion models:
\begin{align*}
    \nabla_{\rvx}\log p_\sigma(\rvx) \approx \E_{\rvy_1 \sim p_{0|t}(\cdot|\rvx)}\E_{\{\rvy_i\}_{i=2}^{n} \sim p^{n-1}}\left[{\sum_{k=1}^n \frac{p_{t|0}(\rvx|\rvy_k)}{\sum_{j} p_{t|0}(\rvx|\rvy_j)} }\nabla_{\rvx}\log p_{t|0}(\rvx|\rvy_k)\right]
\end{align*}
where they sample a large batch of samples $\{\rvy_i\}_{i=2}^n$ from the data distribution to approximate the score function at $\rvx$. They show that the new target can enhance the stability of converged models in different runs/seeds. PFGM++ can be trained in a similar fashion by replacing the target $\frac{{\rvx}-{\rvy}}{{r}/{\sqrt{D}}}$ in perturbation-based objective~(\Eqref{eq:pfgmpp-obj}) with
\begin{align*}
    \frac{1}{{r}/{\sqrt{D}}}\left(\rvx - \E_{p_{0|r(\rvy|\rvx)}}[\rvy]\right) \approx \frac{1}{{r}/{\sqrt{D}}}\left(\rvx - \E_{\rvy_1 \sim p_{0|r}(\cdot|\rvx)}\E_{\{\rvy_i\}_{i=2}^{n} \sim p^{n-1}}\left[{\sum_{k=1}^n \frac{{1}/{({\|\rvx-{\rvy_k}\|_2^2+ r^2})^\frac{N+D}{2}}}{\sum_j{1}/{({\|\rvx-{\rvy_j}\|_2^2+ r^2})^\frac{N+D}{2}}} }\rvy_k\right]\right)
\end{align*}
When $n=1$, the new target reduces to the original target. Similar to \cite{Xu2023StableTF}, one can show that the bias of the new target together with its trace-of-covariance shrinks to zero as we increase the size of the large batch. This new target can alleviate the variations between random seeds. With the new STF-style target, Table~\ref{tab:cifar-stf} shows that when setting $D=3072000\gg N=3072$, the model obtains the same FID score as the diffusion models~(EDM~\cite{Karras2022ElucidatingTD}). It aligns with the theoretical results in Sec~\ref{sec:diffusion}, which states that PFGM++ recover the diffusion model when $D\to \infty$. 
\begin{table}[htbp]
    \small
    \centering
    \caption{FID and NFE on CIFAR-10, using the Stable Target Field~\cite{Xu2023StableTF} in training objective.}
    \begin{tabular}{l c c c}
    \toprule
         &FID $\downarrow$ & NFE $\downarrow$\\
        \midrule
        $D=3072000$  &1.90 & 35\\
        $D\to \infty$~\cite{Karras2022ElucidatingTD}  &1.90 & 35\\
        \bottomrule
    \end{tabular}
    \label{tab:cifar-stf}
\end{table}

\subsection{Extended CIFAR-10 Samples when varying $\alpha$}
\label{app:robust}

To see how the sample quality varies with $\alpha$, we visualize the generative samples of models trained with $D\in \{64, 128, 2048\}$ and $D\to \infty$. We pick $\alpha \in \{0, 0.1, 0.2\}$. \Figref{fig:robust_vis} shows that the smaller $D$s produce better samples compared to larger $D$. Diffusion models~($D\to\infty$) generate noisy images that appear to be out of the data distribution when $\alpha=0.2$, in contrast to the clean images by $D=64, 128$.
\begin{figure*}
    \centering
\subfigure[$D{=}64, \alpha=0$~(FID=1.96)]{\includegraphics[width=0.26\textwidth]{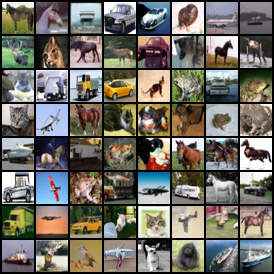}}\hfill
\subfigure[$D{=}64, \alpha=0.1$~(FID=1.97)]{\includegraphics[width=0.26\textwidth]{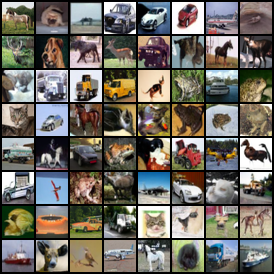}}\hfill
\subfigure[$D{=}64, \alpha=0.2$~(FID=2.07)]{\includegraphics[width=0.26\textwidth]{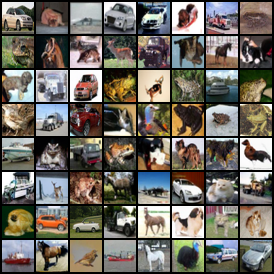}}

\subfigure[$D{=}128, \alpha=0$~(FID=1.92)]{\includegraphics[width=0.26\textwidth]{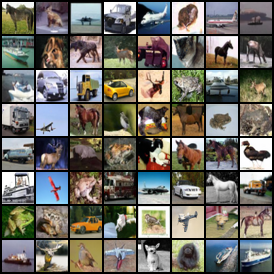}}\hfill
\subfigure[$D{=}128, \alpha=0.1$~(FID=1.95)]{\includegraphics[width=0.26\textwidth]{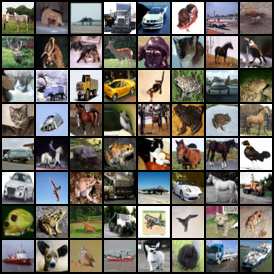}}\hfill
\subfigure[$D{=}128, \alpha=0.2$~(FID=2.19)]{\includegraphics[width=0.26\textwidth]{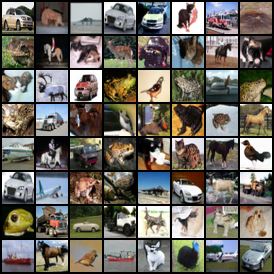}}

\subfigure[$D{=}2048, \alpha=0$~(FID=1.92)]{\includegraphics[width=0.26\textwidth]{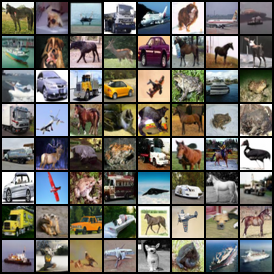}}\hfill
\subfigure[$D{=}2048, \alpha=0.1$~(FID=1.95)]{\includegraphics[width=0.26\textwidth]{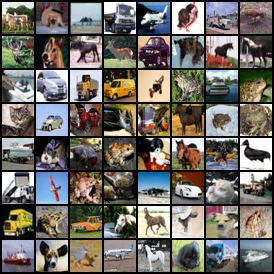}}\hfill
\subfigure[$D{=}2048, \alpha=0.2$~(FID=2.19)]{\includegraphics[width=0.26\textwidth]{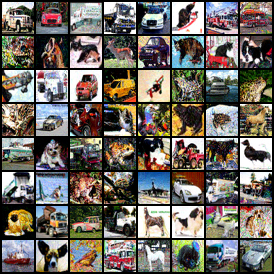}}

\subfigure[$D\to\infty, \alpha=0$~(FID=1.98)]{\includegraphics[width=0.26\textwidth]{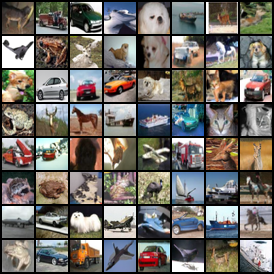}}\hfill
\subfigure[$D\to\infty, \alpha=0.1$~(FID=9.27)]{\includegraphics[width=0.26\textwidth]{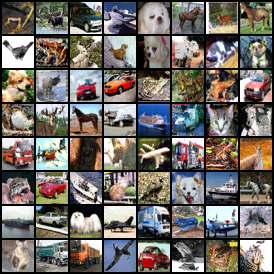}}\hfill
\subfigure[$D\to\infty, \alpha=0.2$~(FID=92.41)]{\includegraphics[width=0.26\textwidth]{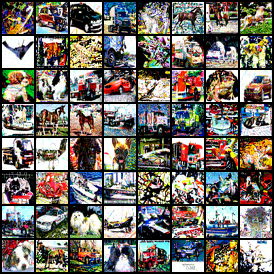}}
    \caption{Generated samples on CIFAR-10 with varied hyper-parameter for noise injection~($\alpha$). Images from top to bottom rows are produced by models trained with $D=64/128/2048/\infty$. We use the same random seeds for finite $D$s during image generation.}
    \label{fig:robust_vis}
\end{figure*}

\subsection{Extended FFHQ Samples}

In \Figref{fig:ffhq-vis}, we provide samples generated by the $D=128$ case and EDM~(the $D\to \infty$ case).
\begin{figure*}
\centering
    \subfigure[$D=128$~(FID=2.43)]{\includegraphics[width=0.45\textwidth]{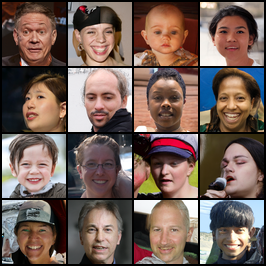}}\hfill
    \subfigure[EDM~($D\to \infty$)~(FID=2.53)]{\includegraphics[width=0.45\textwidth]{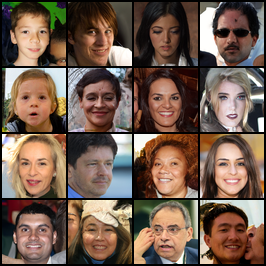}}
    \caption{Generated images on FFHQ $64\times 64$ dataset, by \textbf{(left)} $D=128$ and \textbf{(right)} EDM~($D\to \infty$).}
    \label{fig:ffhq-vis}
\end{figure*}

\section{Potential Negative Social Impact}
\label{app:impact}

The deep generative model is a burgeoning field and has significant potential for shaping our society. Our work presents a novel family of generative models, the PFGM++, which subsume previous high-performing models and provide greater flexibility. The PFGM++ have many potential applications, particularly in areas that require both robustness and high-quality output. However, it is important to note that the usage of these models can have both positive and negative implications, depending on the specific application. For instance, the PFGM++ can be used to create realistic image and audio samples, but it can also contribute to the development of deepfake technology and potentially lead to social scams. Additionally, the data-collecting process for generative models may infringe upon intellectual property rights. To address these concerns, further research is needed to provide robustness guarantees for generative models and to foster collaborations with experts in socio-technical fields.

\end{document}